\documentclass[runningheads]{llncs}

\usepackage{eccv}

\usepackage{eccvabbrv}

\usepackage{graphicx}
\usepackage{booktabs}

\usepackage[accsupp]{axessibility}  

\usepackage{booktabs}
\usepackage{tabularx}
\usepackage{array}
\usepackage{arydshln}
\usepackage{wrapfig}
\usepackage{threeparttable}
\usepackage{float}

\newcolumntype{Y}{>{\raggedright\arraybackslash}X}
\newcolumntype{C}[1]{>{\centering\arraybackslash}p{#1}}
\newcolumntype{L}[1]{>{\raggedright\arraybackslash}p{#1}}

\usepackage[breaklinks,colorlinks,citecolor=eccvblue]{hyperref}

\usepackage{orcidlink}

\usepackage{url}
\usepackage{kotex}
\usepackage{multirow}
\newcommand{\red}[1]{\textcolor{red}{#1}}

\begin{document}


\title{Breaking High Confidence: Practical Face Impersonation under High-Security Thresholds}

\titlerunning{Practical Face Impersonation under High-Security Thresholds}

\author{Changjin Kim\orcidlink{0009-0008-5829-428X} \and
Seunghun Paik\orcidlink{0000-0003-1105-1607} \and
Dongsoo Kim\orcidlink{0009-0007-4341-873X} \and
Jae Hong Seo\thanks{Corresponding author.}\orcidlink{0000-0003-0547-5702}
}

\authorrunning{C.~Kim et al.}

\institute{Department of Mathematics \& Research Institute for Natural Sciences \\
Hanyang University, Seoul 04763, Republic of Korea
\email{\{changjinkim,whiteseoonguh,frds37,jaehongseo\}@hanyang.ac.kr}\\
}

\maketitle

\begin{abstract}

Face recognition systems (FRSs) are increasingly deployed in critical real-world services for authentication, such as banking applications and airport identity checks, necessitating stringent security configurations. Consequently, the security vulnerabilities of FRSs have garnered significant attention. While existing studies have extensively explored FRS security, prior analyses have primarily focused on medium-security threshold settings, which are not directly applicable to FRSs operating under high-security constraints. In this paper, we propose the first successful impersonation attack against FRSs under high-security threshold settings. Among various threat models, we focus on a practical and challenging scenario: score-based impersonation attacks under strict rate limits. To precisely evaluate the feasibility of such attacks, we provide a principled mathematical analysis characterizing the gaps in each stage of the attack pipeline. Our method significantly enhances impersonation capabilities in score-based attacks, even under elevated decision thresholds. On the LFW benchmark, with a budget of only 100 confidence score queries per identity, our attack achieves an impersonation success rate exceeding 92\% against Amazon Rekognition at a confidence score threshold of 99—recommended setting for law enforcement scenarios. We further observe consistently robust performance across multiple open-source FRSs evaluated at similarly stringent decision thresholds.
\end{abstract}


\section{Introduction}\label{sec:intro}

Biometric recognition has become a common way to verify identity in everyday life, and its convenience has driven rapid progress and widespread adoption across diverse applications. 
In particular, face recognition is now a widely used authentication method for identity verification across a broad range of services beyond personal devices to public infrastructure—for example, U.S. agencies deploy facial biometrics in airport identity checks~\cite{tsa_facial_comparison_factsheet}.
It is also increasingly used in financial onboarding, such as remote account opening and eKYC verification~\cite{eba_remote_onboarding_2022}.


As deployments expand into high-stakes domains discussed above, reducing false matches becomes critical (often at the cost of rejecting more legitimate users). 
Thus, practical deployments in such settings frequently adopt conservative decision thresholds. 
In a face recognition system (FRS), decision thresholds are set based on a false match rate (FMR), which means the ratio of non-matching attempts that are incorrectly accepted as matches, and adversaries against FRSs that aim to induce false matches are referred to as impersonation attacks. 
As a representative benchmark for high-security decision thresholds, NIST FRTE (1:1 verification) reports performance using $\mathrm{FMR}=10^{-6}$~\cite{nist_frte_11}. In commercial API settings, AWS Rekognition’s public-safety (law-enforcement) guidance recommends using a high confidence threshold ($99\%$ or higher)~\cite{aws_public_safety_guidance}. 

Impersonation attacks against FRSs have been widely studied.
They are commonly categorized in terms of the feedback type from the target FRS through querying faces: template-based attacks~\cite{razzhigaev2020black,mai2018reconstruction,dai2026clip,dong2023reconstruct,jung2024face} that obtain a feature template from queries, and score-based attacks~\cite{razzhigaev2021darker,razzhigaev2025inverting,park2023towards,kim2024scores} that obtain similarity feedback (e.g., cosine similarity or confidence score) with the enrolled face.
These attacks have demonstrated effectiveness against various FRSs, including even commercial APIs such as Face++~\cite{FacePlusPlus}, Kairos~\cite{Kairos}, and AWS~\cite{AWSRekognition}.
Typical commercial APIs return the confidence score in response to queries, without publicly disclosing any internal structure or formula for computing the confidence score.
Under such settings, score-based attacks in the black-box setting---i.e., no internal knowledge of the target FRS is available to the adversary---have gained attention as a plausible threat against real-world FRS deployments.

\begin{table*}[t]
\centering
\caption{Summary of prior attacks and settings. All attack success rate (ASR) values are computed using the LFW evaluation protocol. We denote by AWS($k$) the AWS Rekognition confidence score threshold of $k$ between 0-100.}
\label{tab:prior_vs_ours}
\vspace{-10pt}
{\scriptsize
\setlength{\tabcolsep}{3.5pt}
\renewcommand{\arraystretch}{1.08}

\begin{tabularx}{\textwidth}{L{2.7cm} C{1.5cm} C{1.3cm} C{2.2cm} C{3.5cm} C{0.05cm} Y}
\toprule
Method & Access & Query budget & ASR@FMR (Mean Cos. Sim.) &
Notes \\
\midrule

Shahreza et al.~\cite{shahreza2022face} & white-box & -- & $94.90@\red{10^{-3}}$ \\
Shahreza et al.~\cite{shahreza2023comprehensive} & white-box & -- & $96.40@\red{10^{-3}}$ & 3D face reconstruction \\
Shahreza et al.~\cite{otroshi2023face} & white-box & -- & $96.40@\red{10^{-3}}$ & high-resolution \\
Shahreza et al.~\cite{shahreza2024vulnerability} & white-box & -- & $96.48@\red{10^{-3}}$ \\
\midrule

\multicolumn{7}{l}{\hspace{-5pt}\textbf{Template-based}} \\
Mai et al.~\cite{mai2018reconstruction} & black-box & -- & $95.20@\red{10^{-3}}$ \\
Dai et al.~\cite{dai2026clip} & black-box & -- & $94.35@\red{10^{-3}}$ & high-resolution \\
Dong et al.~\cite{dong2023reconstruct} & black-box & -- & $98.00@\red{10^{-3}}$ \\
Jung et al.~\cite{jung2024face} & black-box & -- & $73.04@\red{10^{-4}}$  \\
\hdashline
Razzhigaev et al.~\cite{razzhigaev2020black} & black-box & \red{300,000} & $(0.90)$ &  \\

\midrule
\multicolumn{7}{l}{\hspace{-5pt}\textbf{Score-based}} \\
Razzhigaev et al.~\cite{razzhigaev2021darker} & black-box & \red{50,000} & $(0.98)$ &  \\
Razzhigaev et al.~\cite{razzhigaev2025inverting} & black-box & \red{20,000} & $(0.97)$ &  \\
Park et al.~\cite{park2023towards} & black-box & \red{4,000} & $96.25@\red{10^{-3}}$ \\
Kim et al.~\cite{kim2024scores} & black-box & 100 & $99.60@0$ & \red{$13.9@\textrm{AWS}(90)$}\\ \midrule
\textbf{Ours} & black-box & 100 & \textcolor{blue}{$92.8@\textrm{AWS}(99)$} \\
\bottomrule
\end{tabularx}
}
\vspace{-15pt}
\end{table*}

Despite these advancements, however, we find that prior works have mainly analyzed medium-security settings (e.g., $\mathrm{FMR}=10^{-3}$), and the analysis for high-security regimes (e.g., $\mathrm{FMR}=10^{-6}$) has been less explored.
Table~\ref{tab:prior_vs_ours} provides representative prior work in terms of feedback type, query budget, and the attack success rate (ASR) at the given FMR. 
We can observe that all these works either require unrealistically large query budgets (4,000---300,000) for realistic deployments with rate limits, or the ASR sharply drops at high-security threshold settings.
Such a landscape leaves a critical yet underexplored regime on the security analysis of FRSs in high-security settings, despite the increasing adoption of real-world FRS deployments under such scenarios.

\subsection{Our Contributions}

In this paper, we analyze the security of FRSs in high-security settings (i.e., $\mathrm{FMR}=10^{-6}$ or a confidence score threshold of 99) by proposing a black-box score-based attack, revealing their vulnerability for the first time under a practical query budget of at most 100.
We first systematically analyze existing score-based attacks that operate under low query budgets, identifying multiple bottlenecks that limit ASRs in stringent threshold settings.
Through a theoretical analysis of the geometry of the template space, we propose a series of techniques to address these bottlenecks, thereby achieving a non-trivial ASR in various open-source and commercial FRSs.
Notably, using only 100 black-box score queries, our attack achieves an ASR of 92.8\% against the AWS Rekognition API~\cite{AWSRekognition} at a threshold of 99 on the LFW dataset~\cite{huang2008labeled}, where prior work failed to achieve even a non-negligible ASR under such a strict query budget.
Our findings suggest that system-level high-security configurations alone may not suffice to ensure security in adversarial settings, highlighting the need for a better understanding of the inherent robustness of FRSs in such high-security regimes. 
We summarize our contributions as follows:

\begin{itemize}
    \item We conduct the first security analysis of FRSs under high-security thresholds, focusing on practical black-box score-based impersonation.

    \item We revisit prior query-efficient score-based attacks, demonstrating why it fails under high-security settings through systematic analyses.

    \item By analyzing geometric properties of the template space, we propose multiple techniques to enable such attacks in high-security settings, whose effectiveness is supported by both theoretical and experimental analyses.
    
    \item We conduct extensive experimental analyses, revealing vulnerabilities of various open-source and commercial FRSs in high-security threshold settings using at most 100 queries.
    Our attack achieves high ASRs of 92.8\% (LFW~\cite{huang2008labeled}), 95.24\% (CFP-FP~\cite{sengupta2016frontal}), and 96.87\% (AgeDB~\cite{moschoglou2017agedb}) on AWS Rekognition at a confidence score threshold of 99.
    It also achieves 32.63\%--89.27\% on open-source FRSs at $\mathrm{FMR}=10^{-6}$ thresholds.
\end{itemize}

\section{Related Work}


\subsubsection{Score-based Attacks}
%
To impersonate the enrolled identity from the target system's score responses, several studies consider gradient-free optimization techniques, such as hill-climbing~\cite{galbally2010vulnerability,razzhigaev2020black,razzhigaev2021darker,vendrow2021realistic}, genetic algorithms~\cite{dong2023reconstruct,jung2024face,an2023imu,sharif2016accessorize}, or gradient estimation~\cite{park2023towards,razzhigaev2025inverting}.
Recent studies exploit the generative model's latent space to recover high-fidelity facial images~\cite{vendrow2021realistic,dong2023reconstruct,park2023towards,an2023imu,jung2024face}.
While these attacks recover high fidelity as the number of score queries increases, they require a huge number of queries (4,000--300,000).
Furthermore, as demonstrated in several studies~\cite{li2022blacklight,park2025mind}, their adaptive query trajectories often leave structured footprints that are detectable by the target FRS service provider, enabling system-level defenses such as query filtering or request rejection.

Notably, a recent study~\cite{kim2024scores} demonstrated that 100 non-adaptive score queries suffice to achieve a non-trivial ASR on various commercial APIs, which is based on the geometry of template space.
However, their experimental setting is about scenarios with moderate security, e.g., a default threshold in AWS; for thresholds of high-stakes scenarios, their ASR drastically diminishes to 0.

\subsubsection{Template-based Attacks}
From the success of the NbNet~\cite{mai2018reconstruction} in the black-box attack setting, several template inversion attacks have been proposed with various methodologies.
They often assume that the adversary can query the chosen face to the target FRS, obtaining the template extracted from the face.
Under this assumption, they train the inverse model that reconstructs faces from the templates.
Earlier attacks tried to build such a model directly~\cite{mai2018reconstruction, duong2020vec2face, truong2022vec2face}.
Recently, the advancement of generative models, e.g., StyleGANs~\cite{karras2019style,karras2020analyzing} or diffusion models~\cite{ho2020denoising}, facilitates the attack, enabling the adversary to recover high-fidelity facial images from these models. 
Several studies attempted to leverage the latent space of generative models, e.g., training an adaptor network from the template to the latent vector~\cite{otroshi2023face,shahreza2025face,dai2026clip}.
However, they require the adversary to obtain compromised templates enrolled in the target FRS, e.g., through a data breach, which is considered a stronger assumption than score-based attacks.

\section{Backgrounds and Problem Formulation}


%

\subsection{Face Recognition}
Typical FRSs employ the feature extractor that returns a fixed-length feature vector---called \textit{face template}---from the input face.
The feature extractor serves as a distance-preserving mapping, i.e., faces from the same identity are mapped to face templates close to each other, or vice versa.
FRS determines whether two facial images represent the same identity by computing the similarity score between templates; if it exceeds a threshold, they are recognized as the same person.
Recent FRSs~\cite{schroff2015facenet,deng2019arcface,meng2021magface,kim2022adaface,boutros2022elasticface,kim2024keypoint,you2025lvface} represent face templates as a unit vector and measure their similarity scores through cosine similarity.

As the demand for face recognition-based applications increases, several commercial API services, such as AWS Rekognition~\cite{AWSRekognition}, Tencent~\cite{TencentCloud}, Kairos~\cite{Kairos}, and Face++~\cite{FacePlusPlus}, have been adopted in various FRS deployments~\cite{AWSUseCases,TencentUseCases}.
The user can enroll the face in these services, and then query the face to obtain the confidence score with the previously enrolled one.
Confidence scores are used for determining whether the enrolled and queried faces are from the same person, and the formula for computing the confidence score is publicly unknown.


\subsection{Score-based Face Impersonation Attack}

We consider two entities, the adversary and the target FRS where the target face is enrolled.
Following prior score-based impersonation attacks~\cite{vendrow2021realistic,razzhigaev2020black,razzhigaev2021darker,razzhigaev2025inverting,dong2023reconstruct,kim2024scores}, the adversary can query its own face to the target FRS and obtain a confidence score.
We consider the \textit{black-box} scenario; the adversary does not know the internal information of the target FRS, including the parameters, architecture, and training data of the feature extractor, and the decision threshold.
Any prior information on the enrolled identity is unknown to the adversary.
In practical scenarios with commercial APIs, the number of score queries made by the adversary is limited due to API rate limits and financial costs~\cite{ilyas2018black,kim2024scores}.
Hence, we assume a query-limited adversary with a small query budget, e.g., tens to hundreds.
Under this setting, the adversary aims to craft a facial image that is recognized as the same identity.



\section{Revisiting Prior Low-Query Score-based Attacks}
Motivated by this limitation, we revisit a recent score-based attack under a limited query budget~\cite{kim2024scores} as our baseline attack method.
We systematically analyze several errors involved during the attack and identify why the attack becomes less effective when attacking FRSs under stringent threshold settings.


\subsection{High-Level Idea of the Attack}\label{sec:highlevel}

The core insight of their attack is that well-trained face recognition models behave as almost isometries, i.e., their embedding spaces exhibit similar geometric structures in terms of pairwise similarity.
They empirically verify this assumption across various open-source FRSs.
They also observed that each score query reveals the target template's relative position within the embedding space.
Thus, given a surrogate face recognition model, the adversary can estimate the target template's position in the surrogate model's embedding space with score queries.

Precisely, if we denote $s_{i}$ as the score from $i^{\text{th}}$ face image query, then one can write $s_{i} = \langle x_{i}, y \rangle +\epsilon_{i}$ for the templates $x_{i}, y \in \mathbb{S}^{d-1}$ of queried and enrolled faces, respectively.
Here, $\epsilon_{i}$ indicates the error inherited from the mismatch between the target and surrogate models' embedding spaces, which is unknown to the adversary.
That is, after $Q$ queries, the adversary obtains an equation $Ay = s-\epsilon$, where $A$ is a $Q \times d$ matrix whose $i^{\text{th}}$ row is $x_{i}$ and $s=(s_{i})_{i=1}^{Q},\epsilon=(\epsilon_{i})_{i=1}^{Q} \in \mathbb{R}^{Q}$.
Since this equation is underdetermined and the errors are unknown, the adversary utilizes the least square method to estimate $y$ while ignoring errors, i.e., computing $\widehat{y} \gets A^{\dagger}s$, where $A^{\dagger}$ denotes the pseudoinverse of $A$.
Finally, the adversary can recover the candidate facial image of the enrolled identity by inverting $\widehat{y}$ through template inversion methods, i.e., inverse models, for the surrogate model~\cite{mai2018reconstruction,duong2020vec2face,otroshi2023face,shahreza2025face,shahreza2024template}. 
A detailed algorithm is provided in Section~A.1 of the supplementary material.


\subsubsection{Key Techniques to Instantiate the Attack}
To realize the attack, they proposed two techniques.
The first one is the orthogonal face set (OFS), a set of facial images whose templates are (almost) pairwise orthogonal in the embedding space.
They showed that the orthogonality helps reduce the effect of errors $\epsilon_{1},\dots,\epsilon_{Q}$ after being multiplied by $A^{\dagger}$.
In addition, since commercial APIs return confidence scores rather than cosine similarities, they propose an estimation method of the cosine similarity from the confidence score via numerical methods.

\subsection{Investigating Errors from Approximations}\label{sec:investigateerrors}

While elegant, the aforementioned approach relies on several approximations.
Specifically, for high-threshold settings, the attack's efficacy becomes much more sensitive to these approximations, as the face reconstructed by the adversary must be accurate enough to surpass the tight decision threshold.
For this reason, we investigate and quantify the impact of these approximations on the precision of the recovered face and their consequences for the success of the attack.


%
Throughout the attack pipeline, the error occurs when (i) interpreting score queries as cosine similarities (\textbf{E1}), (ii) solving the system of linear equations through pseudoinverse (\textbf{E2}), and (iii) the reconstruction error from the template inversion (\textbf{E3}).
To formalize this, let us denote $z$ as the template of the final reconstructed template, and $A$, $y$, $s$, $\widehat{y}=A^{\dagger}s$, and $\epsilon$ follow the same notation as in Section~\ref{sec:highlevel}.
Recall that $Ay = s - \epsilon$.
Then we have that 
\begin{align}\label{eq:errors}
    \|y - z\|_{2} \le \underbrace{\|y - A^{\dagger}(s-\epsilon) \|_{2}}_{(\textbf{E2})} + \underbrace{\|A^{\dagger}(s-\epsilon) - \widehat{y} \|_{2}}_{(\textbf{E1})} + \underbrace{\|\widehat{y} - z\|_{2}}_{(\textbf{E3})}.
\end{align}

We now analyze each error as follows.
First, since $\widehat{y} = A^{\dagger}s$, the error (\textbf{E1}) corresponds to $\|A^{\dagger}\epsilon\|_{2}$.
If we denote the operator norm of $A^{\dagger}$, denoted by $\|A^{\dagger}\|$, then $\|A^{\dagger}\epsilon\|_{2} \le \|A^{\dagger}\| \cdot \|\epsilon\|_{2} = \sigma_{\min}(A)^{-1} \cdot \|\epsilon\|_{2}$,
where $\sigma_{\min}(A)$ denotes the smallest singular value of $A$.
This observation is essentially identical to the analysis by~\cite{kim2024scores} based on the condition number, which justifies the use of OFS.
The error (\textbf{E2}) corresponds to the gap between $y$ and its projection onto the row space of $A$.
Finally, the error (\textbf{E3}) depends on the template inversion method.

\subsubsection{Quantifying the Errors in~\cite{kim2024scores}}
%
To quantify the errors (\textbf{E1})--(\textbf{E3}), we design the following experiments that measure the cosine similarity between the target and reconstructed faces.
We follow the same attack algorithm in~\cite{kim2024scores}.
\begin{itemize}
    \item (Exp. 1) Black-box attack: errors (\textbf{E1})-(\textbf{E3}) affect.

    \item (Exp. 2) White-box attack: errors (\textbf{E2}) and (\textbf{E3}) affect.
    
    \item (Exp. 3) White-box attack with the ideal inversion: only error (\textbf{E2}) affects.

    \item (Exp. 4) White-box attack with the ideal inversion + PCA subspace: error (\textbf{E2}) remains, but its amount is minimized due to the optimality of the PCA subspace in terms of (average) projection gap.
\end{itemize}

We use the same surrogate and inverse models as in~\cite{kim2024scores}, and for the black-box attack, we select the ViT-based FRS~\cite{kim2024keypoint} as the target model\footnote{The surrogate and target models correspond to $F_{\mathsf{S}}$ and $F_{1}$ in Tab.~\ref{tab:summary(FRSs)}, respectively.}, whose training dataset, architecture, and loss differ from the surrogate model.
We run the PCA on the embeddings of the MS1MV3 dataset~\cite{deng2020retinaface} extracted from the surrogate model and select the top 100 components with the highest singular values. 
We use the LFW dataset~\cite{huang2008labeled} as the target enrolled faces.
To simulate the high-threshold scenario, we provide an estimate of the threshold for $\mathrm{FMR}=10^{-6}$.

\begin{figure}[t]
    \centering
    \includegraphics[width=.95\linewidth]{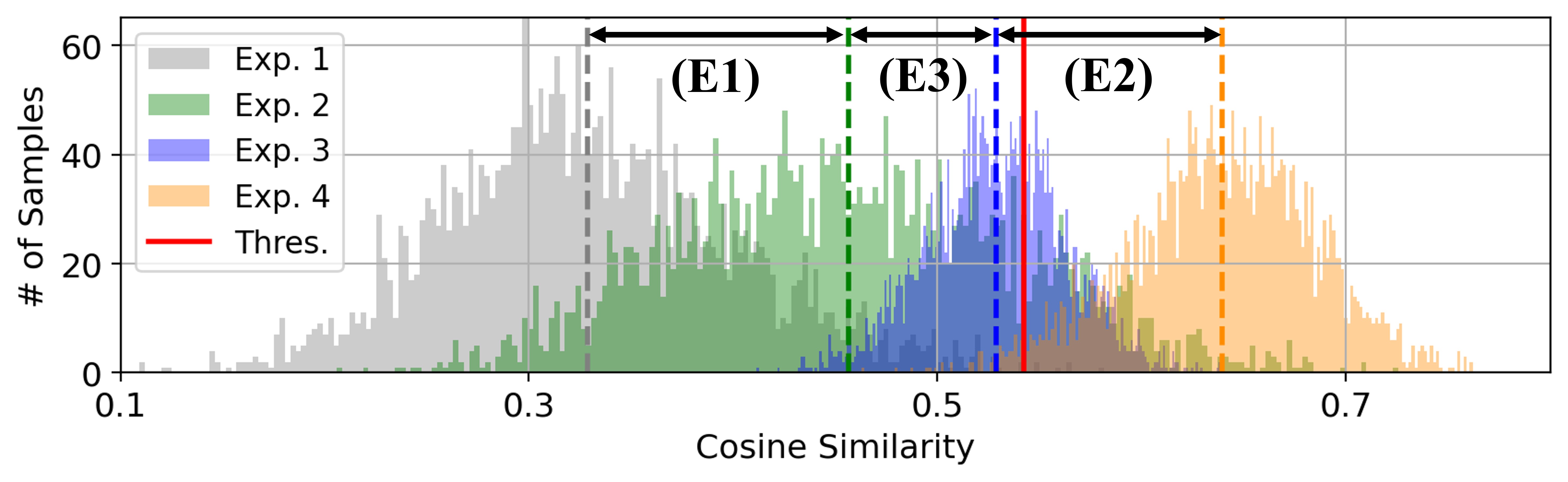}
    \vspace{-12pt}
    \caption{
    The errors (\textbf{E1})-(\textbf{E3}) in the attack by~\cite{kim2024scores} measured through experiments (Exp. 1--4).
    The dashed lines indicate the average of each distribution. The solid red line denotes the estimated threshold when $\mathrm{FMR}=10^{-6}$. Best viewed in color.
    }
    \label{fig:whyfails}
    \vspace{-10pt}
\end{figure}

The results are visualized in Fig.~\ref{fig:whyfails}. 
For the baseline case (Exp. 1), almost all samples fail to exceed the threshold line of $\mathrm{FMR}=10^{-6}$ (red solid line) due to errors, except for only a small portion of samples ($\approx 0.2\%$).
However, when we sequentially remove the effect of errors, the distribution of cosine similarity values gradually moves toward the right, and notably, for the rightmost case (Exp. 4), more than 97\% of samples exceed the threshold.
The averages of cosine similarity values for each experiment (Exp. 1--4) are 0.3289, 0.4567, 0.5289, and 0.6392, respectively.
This result indicates that all these errors (\textbf{E1})-(\textbf{E3}) in Eq.~\eqref{eq:errors} intricately yet significantly affect the failure of~\cite{kim2024scores} under high-security threshold settings.
In the following section, we present our techniques to reduce each error, thus addressing the limitation.

\section{Reducing Errors to Surpass High Thresholds}


To overcome the errors (\textbf{E1})-(\textbf{E3}) in Eq.~\eqref{eq:errors}, we revisit the techniques in~\cite{kim2024scores} and show that those errors can be further reduced without additional queries.

\subsection{(E1): Correction Matrix to Tame Metric Distortion}
Motivated by~\cite{kim2025non}, we utilize the \textit{correction matrix} to reduce the error (\textbf{E1}).
The correction matrix is defined as the inverse of the matrix $S=(s_{i,j})$, where $s_{i,j}$ denotes the cosine similarity between the $i^{\text{th}}$ and $j^{\text{th}}$ OFS elements obtained from score queries.
With the correction matrix, the adversary computes $\widehat{y} \gets A^{T}S^{-1}s$ instead of $A^{\dagger}s$.
While~\cite{kim2025non} justified this idea through a somewhat exotic assumption, i.e., OFS serves as a universal basis over faces, we propose a theoretically principled interpretation of the correction matrix in terms of metric distortion.

Recall that $\widehat{y} \gets A^{\dagger}s$ corresponds to solving the least squares over the surrogate embedding space, ignoring the metric distortion characterized by the additive error $\epsilon$.
Instead, if we apply a first-order approximation to such a distortion by introducing a positive definite matrix $\Sigma$, we can incorporate the correction matrix with the (approximated) projection over the target template space.
For templates $x,y$ in the surrogate model's embedding space, let us assume that their cosine similarity in the target model can be approximated by $x^{T} \Sigma y$.
Then we can write $S = A\Sigma A^{T}$, and $s = A\Sigma y$, and $\widehat{y} = A^{T}S^{-1}s = A^{T}(A \Sigma A^{T})^{-1} (A\Sigma )s$.

We can observe that $\widehat{y}$ becomes the projection of $y$ onto the row space of $A$ in the metric determined by $\Sigma$, namely\footnote{The proof is given in Section~B of the supplementary material.}, $\widehat{y} = \arg \min_{z \in R(A)} (y-z)^{T} \Sigma (y-z) \nonumber$, where $R(A)$ denotes the vector space spanned by the rows of $A$.
That is, $\widehat{y}$ provides a better estimate in the target embedding space in terms of the linear combination of OFSs.
While this interpretation relies on approximating the metric distortion via $\Sigma$, we empirically verify that the correction matrix substantially reduces the error (\textbf{E1}), especially when attacking the commercial FRS.

\subsection{(E2): OFS with an Optimal Projection Gap}
While the error (\textbf{E2}) stems from the subspace projection, the baseline OFS by \cite{kim2024scores} enforces orthogonality only.
Instead, we select the OFS obtained from the PCA, which provides the optimal subspace to minimize the projection gap.
More precisely, for a dataset $X \subset \mathbb{R}^{d}$, the PCA with $k$ principal components finds the row orthogonal matrix $W \in \mathbb{R}^{k \times d}$ which minimizes $\frac{1}{|X|}\sum_{y \in X}\|W^{T}Wy - y\|_{2}^{2}$.
Hence, with a sufficiently large dataset for $X$, e.g., large-scale public training datasets, the adversary can effectively reduce the error (\textbf{E2}) by setting $A=W$, i.e., using the faces corresponding to the principal components as the OFS.

\subsection{(E3): Improved Inverse Model}

To reduce the error (\textbf{E3}), we propose a new inverse model called the \textit{SPNet}.
%
Our design starts from the NbNet~\cite{mai2018reconstruction}, the first black-box inverse model.
Recently, \cite{paik2025reversibility} improved the NbNet by adopting the style mapping network of StyleGAN~\cite{karras2019style} with the adaptive instance normalization (AdaIN)~\cite{huang2017arbitrary}, regarding an input face template as a style.
They replaced each batch normalization~\cite{ioffe2015batch} with AdaIN.
In contrast, \cite{shahreza2024vulnerability} pointed out that the NbNet is prone to producing blurry output, proposing a new architecture, called DSCasConv, based on the skip connection.
We find that merging these techniques can achieve the best of both worlds, obtaining a more accurate inverse model.

In addition, we made the following micro-level modifications, including (i) removing \textsf{tanh} at the last layer to obtain a sharper facial image, (ii) replacing \textsf{ReLU} with \textsf{GELU}, and (iii) employing pixel-wise normalization (PNorm) \cite{karras2017progressive} for the deconvolution layer block.
Moreover, by following~\cite{paik2025reversibility}'s observation, we directly use the MS1MV3 training dataset~\cite{deng2020retinaface}.
Though our design choices are empirical, we empirically verify that they substantially improve the inversion precision, as well as the attack success rate.
We provide the detailed architecture configuration in Section~C of the supplementary material.

\subsubsection{Training Losses}
To train the proposed inverse model, we employ the following losses that are widely used for training the inverse model.
\begin{itemize}
    \item (Pixel Loss) it minimizes the absolute error between the original face $x$ and reconstructed ones $\widehat{x}$, namely, $\mathcal{L}_{\mathrm{Pixel}}(x,\widehat{x}) = \|x - \widehat{x}\|_{1}$.

    \item (Identity Loss) it maximizes the cosine similarity between templates of the original face $x$ and reconstructed ones $\widehat{x}$ extracted from various FRSs.
    More precisely, for FRSs $F_{1}, \dots, F_{k}$, we define $\mathcal{L}_{\mathrm{ID}}(x, \widehat{x}) = \sum_{i=1}^{k} (1- \langle F_{i}(x), F_{i}(\widehat{x}) \rangle)$, assuming that each embedding is normalized.
\end{itemize}
Since the adversary trains the inverse model of its own surrogate FRS, unlike black-box inverse models~\cite{mai2018reconstruction,duong2020vec2face}, we also include the surrogate model to compute $\mathcal{L}_{\mathrm{ID}}$.
The final loss is $\mathcal{L}_{\mathrm{Train}} = \lambda_{1}\mathcal{L}_{\mathrm{Pixel}} + \lambda_{2} \mathcal{L}_{\mathrm{ID}}$ for hyperparameters $\lambda_{1},\lambda_{2}$.

\subsection{Effect of the Proposed Techniques on the Errors}
To show the validity of the proposed techniques, we show how they affect the cosine similarity of original and recovered faces in the black-box and white-box settings.
In the black-box attack setting, we enable all the techniques (correction matrix, PCA-based OFS, and SPNet) to reduce all the errors (\textbf{E1}-\textbf{E3}).
In contrast, since the error (\textbf{E1}) is not engaged in the white-box attack setting, we only enable the PCA-based OFS and SPNet. 
All the remaining experimental settings, e.g., choice of the surrogate/target FRSs and the dataset to run the PCA, are equivalent to those of Fig.~\ref{fig:whyfails} of Section~\ref{sec:investigateerrors}.
The results are visualized in Fig.~\ref{fig:compare_wb} and~\ref{fig:compare_bb}, respectively.
We can observe that our techniques substantially improve the cosine similarities in both settings, showing their validity.
Notably, we can observe that these techniques finally enable the similarity values to surpass a high-security threshold; 20\% and 78.07\% of samples exceed the red line in black-box and white-box attacks, respectively.

\begin{figure}[!t]
    \centering
    \begin{subfigure}[t]{0.49\linewidth}
        \centering
        \includegraphics[width=\linewidth]{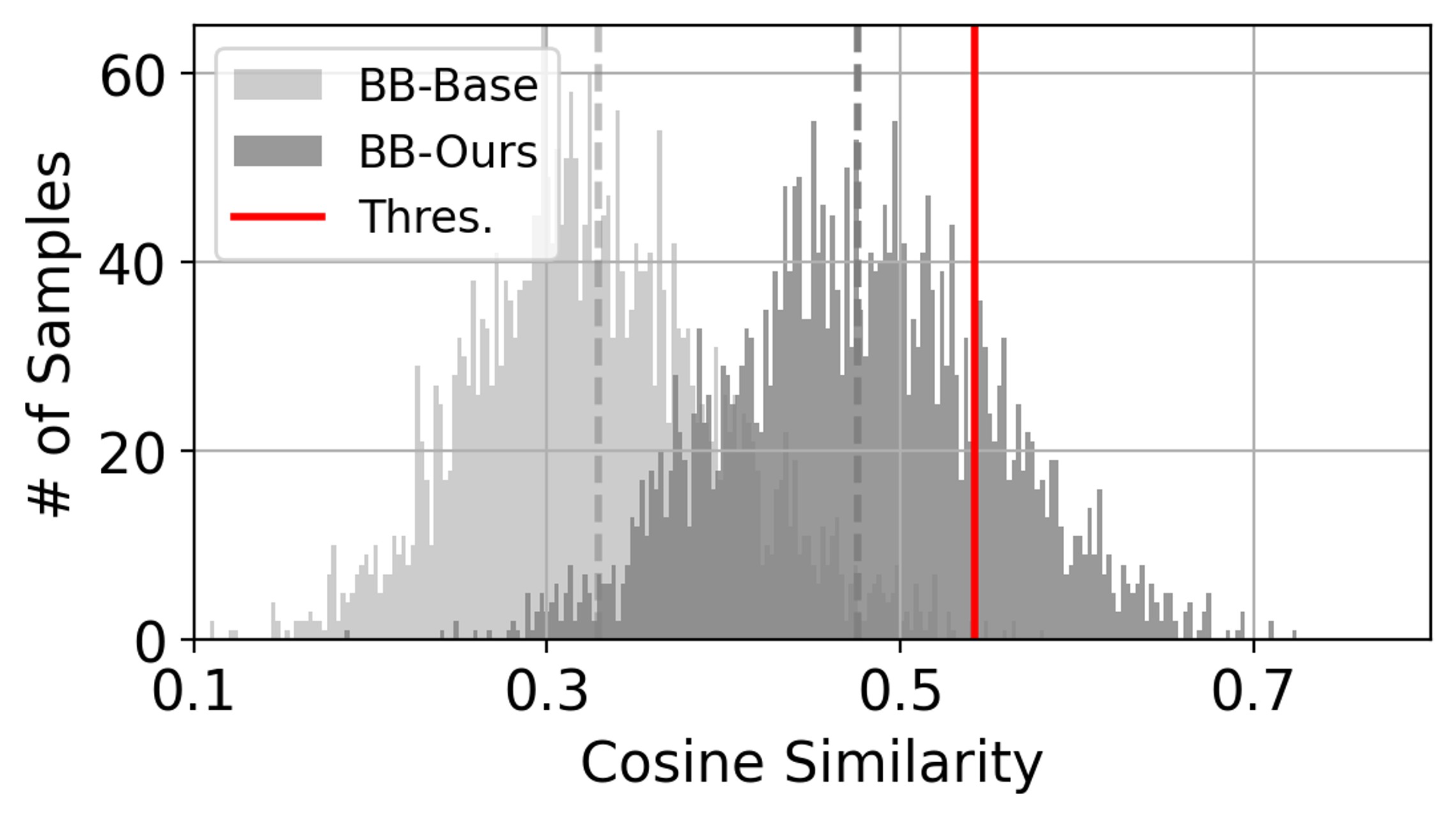}
        \caption{Black-Box Attacks}
        \label{fig:compare_wb}
    \end{subfigure}
    \hfill
    \begin{subfigure}[t]{0.49\linewidth}
        \centering
        \includegraphics[width=\linewidth]{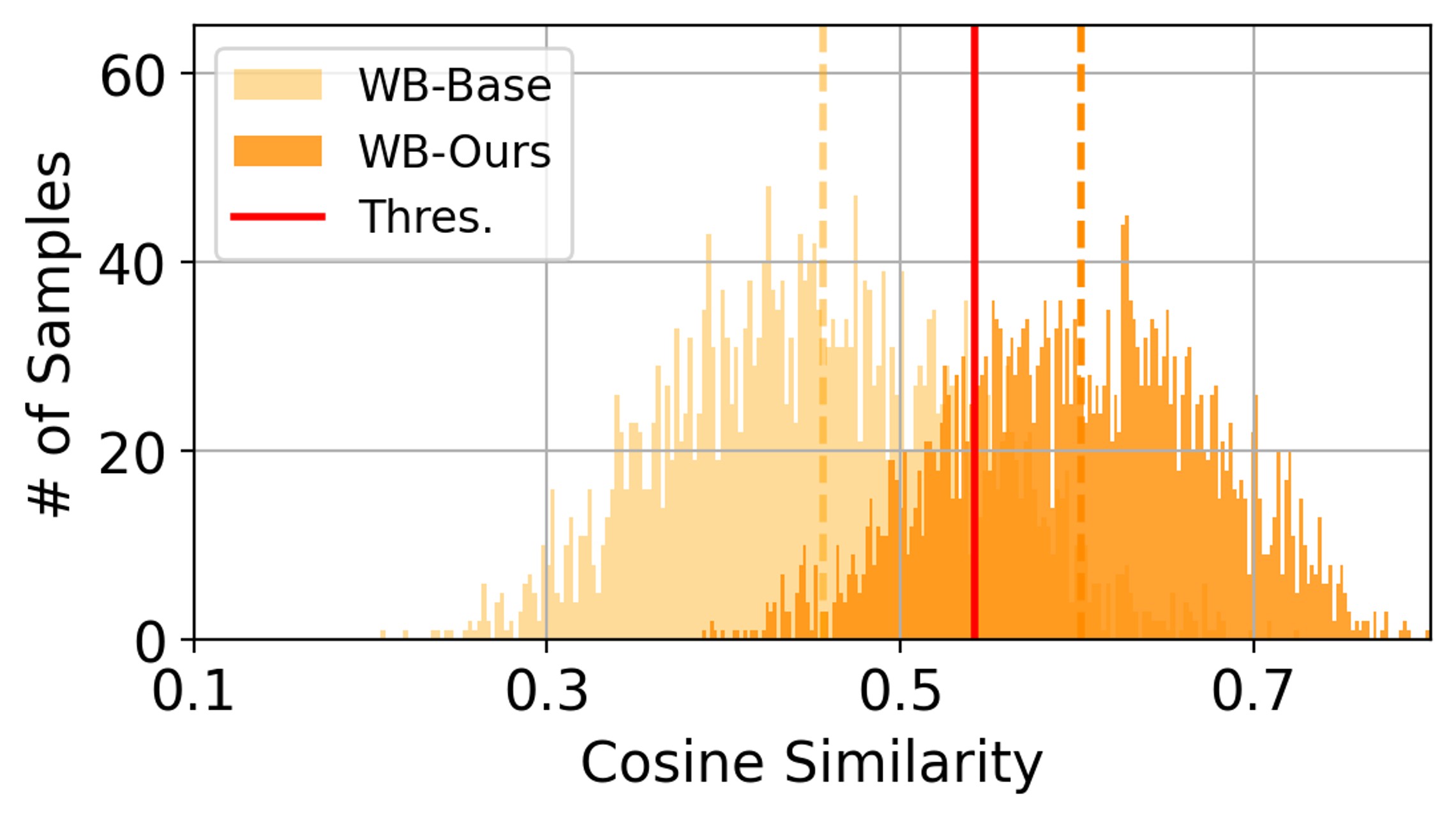}
        \caption{White-Box Attacks}
        \label{fig:compare_bb}
    \end{subfigure}
    \vspace{-5pt}
    \caption{
    The error analysis with our techniques---correction matrix (\textbf{E1}), PCA-based OFS \textbf{(E2)}, and SPNet \textbf{(E3)}---with a comparison to the baseline~\cite{kim2024scores} in the black-box (BB) and white-box (WB) attacks.
    In the BB attack of ours, all the techniques are enabled, whereas only PCA-based OFS and SPNet are enabled in the WB attack.
    }
    \label{fig:haha}
    \vspace{-12pt}
\end{figure}


\section{Experimental Analysis}


To verify the effectiveness of the proposed techniques, we conduct extensive experimental analyses on various open-source and commercial FRSs.

\subsection{Experimental Settings}

\subsubsection{Attack Experimental Setting}
We select the following open-source FRSs as the target systems: ViT-KPRPE~\cite{kim2024keypoint} $(F_{1})$, TopoFR~\cite{dan2024topofr} $(F_{2})$, SphereFace-R~\cite{liu2022sphereface} $(F_{3})$.
Their pre-trained parameters are publicly available in their public GitHub repositories, e.g., \texttt{CVLFace}~\cite{CVLface} and \texttt{OpenSphere}~\cite{opensphere}.
For the surrogate model $(F_{\mathsf{S}})$ held by the adversary, we use the ArcFace~\cite{deng2019arcface} from the \texttt{InsightFace} library~\cite{insightface}.
We select the AWS CompareFace API $(F_{\mathsf{A}})$ as the target commercial system.
We summarize the specification of each FRS in Tab.~\ref{tab:summary(FRSs)}.

We evaluate the attack success rate (ASR) by measuring how many reconstructed faces are authenticated as the target face.
The thresholds for open-source FRSs are set on the FMR of $10^{-6}$.
Since the AWS CompareFace returns the confidence score between 0 and 100, we select 80, 90, and 99 as the decision thresholds; note that 80 is the default threshold~\cite{aws_comparefaces_doc} and 99 is recommended for high-stakes use-cases~\cite{aws_public_safety_guidance}.
For evaluation, we select LFW~\cite{huang2008labeled}, CFP-FP~\cite{sengupta2016frontal}, and AgeDB~\cite{moschoglou2017agedb} datasets, which are widely used for benchmarking the accuracy of FRSs.
These datasets consist of pairs of facial images from the same or different identities; we enroll the first face of the pairs from the same identity and measure the ASR.
Details on these datasets are provided in Section~D of the supplementary material.
To craft the PCA-based OFS, we select 100 principal components of templates of the MS1MV3 dataset~\cite{kim2024scores} extracted from $F_{\mathsf{S}}$.
When attacking $F_{\mathsf{A}}$, we use the same method to convert confidence scores into cosine similarities as~\cite{kim2024scores}, whose details are provided in Section~A.2 of the supplementary material.

\begin{table}[t]
    \centering
    \caption{Summary of target open-source and commercial FRSs. Thres. denotes estimated threshold when $\mathrm{FMR}=10^{-6}$ for each open-source FRS.}
    \vspace{-10pt}
    {\scriptsize
    \setlength{\tabcolsep}{5pt}
    \begin{tabular}{lcccc}
    \toprule
       \textbf{FRS} & \textbf{Architecture} & \textbf{Loss} & \textbf{Train Dataset}  & \textbf{Thres.}  \\ \midrule
       $F_{\mathsf{S}}$~\cite{deng2019arcface}  & ResNet-100~\cite{he2016deep} & ArcFace & Glint360K~\cite{an2022killing} & 0.5422 \\  \midrule
       $F_{\mathsf{1}}$~\cite{kim2024keypoint}  & ViT-Base~\cite{dosovitskiy2020image} & AdaFace~\cite{kim2022adaface}+KPRPE & WebFace12M~\cite{zhu2021webface260m} & 0.5195 \\
       $F_{\mathsf{2}}$~\cite{dan2024topofr}  & ResNet-100  & CosFace~\cite{wang2018cosface}+TopoFR & MS1MV2~\cite{deng2019arcface} & 0.5216 \\ 
       $F_{\mathsf{3}}$~\cite{liu2022sphereface}  & ResNet-100 & SphereFace-R & MS1MV2 & 0.5537  \\  \midrule
       $F_{\mathsf{A}}$~\cite{AWSRekognition}  & \multicolumn{4}{c}{Not publicly disclosed} \\ \bottomrule        
    \end{tabular}
    }
    \label{tab:summary(FRSs)}
    \vspace{-12pt}
\end{table}

\subsubsection{Threshold Estimation for $\mathrm{FMR}={10^{-6}}$}
%
In LFW, CFP-FP, and AgeDB, low FMRs like $10^{-6}$ cannot be directly measured because of insufficient numbers of false pairs.
While the IJB-C benchmark~\cite{maze2018iarpa} provides low FMR settings, e.g., $10^{-6}$, the resulting threshold is not aligned with our evaluation datasets because, unlike them, the templates in the IJB-C protocol are the average of embeddings from multiple faces of the same identity.
Thus, we estimate the threshold for $\mathrm{FMR}=10^{-6}$ by using large-scale face datasets. 
Since $F_{\mathsf{S}}$, $F_{2}$, and $F_{3}$ are trained from MS1M-derived datasets (MS1MV2, Glint360K), to avoid potential overlap, we select the CASIA-WebFace dataset~\cite{yi2014learning} that is known not to be derived from the MS1M lineage.
Similar to the NIST-FRTE protocol~\cite{nist_frte_11}, we randomly sample 4,096 identities, sample one image per identity, and compute the false-pair cosine similarities from all cross-identities to find the candidate threshold for $\mathrm{FMR}=10^{-6}$.
Tab.~\ref{tab:summary(FRSs)} reports the average over 1,000 repetitions.
Though CASIA-WebFace may introduce distributional bias, the resulting ASRs are likely conservative: at directly measurable FMRs $10^{-2}$ and $10^{-3}$, CASIA-WebFace often yields comparable or tighter thresholds than the benchmark-specific ones.
This suggests that the resulting ASRs based on CASIA-WebFace thresholds are not optimistic estimates, and details are provided in Tab.~2 and~3 of Section~D in the supplementary material.


\subsubsection{Training Inverse Model}
We train SPNet with the MS1MV3 dataset~\cite{deng2020retinaface}.
We use the AdamW optimizer~\cite{loshchilov2019decoupled} with an initial learning rate of 0.001, weight decay of 0.0005, and a batch size of 64 for 81,000 total steps (approximately one epoch for the MS1MV3 dataset). 
A linear warm-up is applied during the first 5\% of training steps, followed by cosine annealing decay to 1\% of the initial learning rate. 
The loss coefficients are set to $\lambda_1=\lambda_2= 10$. 
For the identity loss, we employ an ArcFace ResNet-100 trained on MS1MV3 with ElasticFace-Cos+~\cite{boutros2022elasticface}.

\subsection{Attack Results}

\subsubsection{Open-source FRSs}
Fig.~\ref{fig:opensource} visualizes cosine similarity scores of true/false pairs (orange/green histograms) of the LFW dataset and those between the target and reconstructed faces from our attack.
Our attack achieves non-trivial ASRs in all target FRSs, achieving 76.83\% in the white-box setting ($F_{S}$) and 32.63\%--89.27\% in the black-box setting ($F_{1}$--$F_{3}$).
Since all training settings for $F_{1}$ differ from $F_{\mathsf{S}}$, it exhibits the lowest ASR, though it still achieves a non-negligible ASR.
For $F_{2}$ and $F_{3}$, our attack achieves much higher ASR, even exceeding the result in the white-box setting.
This stems from the extent of non-uniformity in the face template distributions, which is discussed in Section~E.1 of the supplementary material.
We provide the results for the CFP-FP and AgeDB datasets in Section~E.2 of the supplementary material.

\begin{figure}[t]
    \centering
    \includegraphics[width=\linewidth]{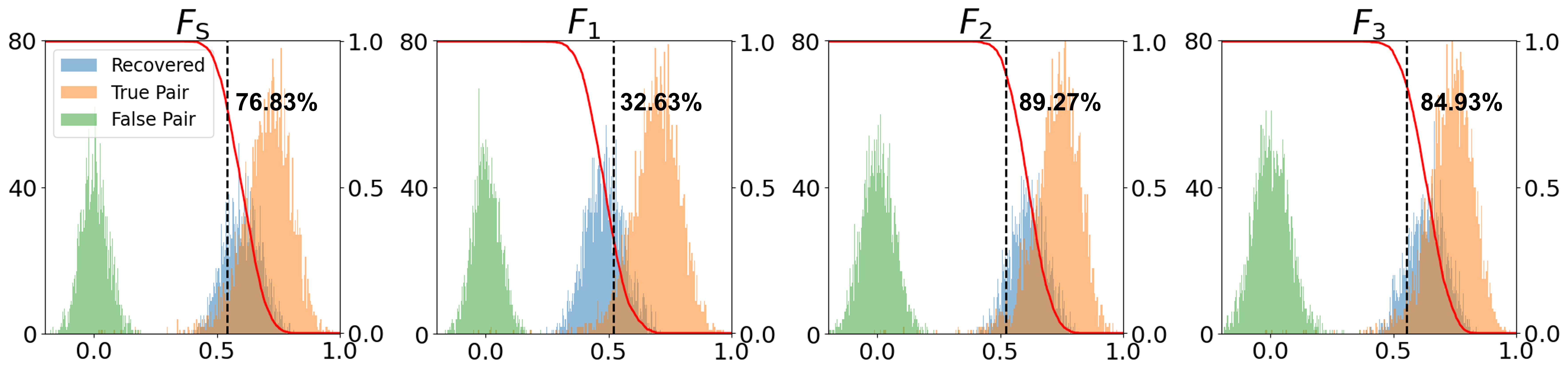}
    \vspace{-22pt}
    \caption{
    Attack results of ours on various open-source FRSs in the LFW dataset. x-axis: cosine similarity. y-axis: the number of samples (left), ASR (right).    
    Red solid lines denote the ASR for each threshold, and black dashed lines denote the threshold at $\mathrm{FMR}=10^{-6}$.
    We also report the ASR at that threshold.
    Best viewed in color.
    }
    \vspace{-14pt}
    \label{fig:opensource}
\end{figure}

\subsubsection{AWS CompareFace}
Tab.~\ref{tab:AWS_impersonation} shows that the baseline attack exhibits an ASR of at most $0.60\%$ in the tight threshold setting (99), though it achieves a non-trivial ASR at moderate thresholds.
When our techniques for (\textbf{E1}) and (\textbf{E2}) are enabled, the attack achieves an ASR of up to 63.50\%.
Such a gain is also observed when we replace NbNet with Arc2Face~\cite{papantoniou2024arc2face}, a diffusion-based inverse model that produces high-resolution ($512 \times 512$) faces.
After replacing NbNet with SPNet, the ASR exceeds 92.80\% even in a high-stakes scenario threshold.
For comparison with~\cite{kim2024scores}, we re-implement its attack and evaluate it under the same AWS CompareFace API version as ours. The reproduced ASR slightly differs from the originally reported ASR.
We guess that this discrepancy may stem from differences in the API/model version. We access the AWS CompareFace API with \texttt{FaceModelVersion: v7} (announced: 2023-12-05; accessed: 2026-03-05), whereas the baseline likely used v6 (announced: 2022-05-06), given the camera-ready deadline (2023-08-18) of the IEEE S\&P 2024 summer cycle.

\begin{table*}[t]
\centering
\setlength{\tabcolsep}{3pt}
\caption{ASRs (\%) on benchmark datasets against $F_{\mathsf{A}}$. Baseline corresponds to ``NbNet'' without our techniques. C: correction matrix. P: PCA-based OFS.}
\vspace{-10pt}
{\scriptsize
\label{tab:AWS_impersonation}
\begin{tabular}{l ccc ccc ccc}
\toprule
\multicolumn{1}{c}{Used Tools/Tech.} & \multicolumn{3}{c}{LFW} & \multicolumn{3}{c}{CFP-FP} & \multicolumn{3}{c}{AgeDB} \\
\cmidrule(lr){2-4} \cmidrule(lr){5-7} \cmidrule(lr){8-10}
& 80 & 90 & 99 & 80 & 90 & 99 & 80 & 90 & 99 \\
\midrule
Baseline$^{\ast}$~\cite{kim2024scores} & 27.0 & 13.9 & --   & 20.6 & 9.5 & -- & 29.6 & 17.7 & --   \\
Reproduced$^{\dagger}$ & 21.13 & 9.27 & 0.33   & 14.91 & 6.49 & 0.14 & 15.53 & 6.6 & 0.17   \\
\midrule 
NbNet+C & 70.89 & 48.98 & 2.83   & 64.14 & 41.43 & 2.46 & 61.17 & 40.27 & 2.10   \\
NbNet+P & 99.60 & 97.37 & 46.68   & 99.89 & 98.4 & 51.26 & 99.73 & 98.47 & 55.67   \\
\midrule 
NbNet+CP             & 99.67 & 98.20 & 56.10 & 99.63 & 98.63 & 60.60 & 99.93 & 99.23 & 63.50 \\
Arc2Face+CP & 99.43 & 97.83 & 53.92 & 99.26 & 97.71 & 53.03 & 99.13 & 96.53 & 44.83 \\
\midrule 
SPNet+CP & 99.90 & 99.83 & \textbf{92.80} & 99.87 & 99.84 & \textbf{95.24} & 100.00 & 100.00 & \textbf{96.87} \\
\bottomrule
\end{tabular}
}
\begin{tablenotes}
\scriptsize
\item[1] $^{\ast}$These results are reported in their original paper.
\item[2] $^\dagger$The difference in the results may stem from the underlying API/model version mismatch.
\end{tablenotes}
\vspace{-15pt}
\end{table*}


\subsubsection{Visualization of Attack Result}
Fig.~\ref{fig:aws_att_img} provides examples of the attack results; the 1st row and 2nd rows are the faces of the true pairs in the LFW dataset. 
The 3rd row shows the reconstruction of target faces in the 1st row.
We also provide the confidence score from $F_{\mathsf{A}}$.
Although the resulting faces are visually dissimilar to the targets, the confidence score is sufficiently high to surpass the high-stakes scenario threshold, even surpassing the score of true pairs in some cases.
This shows the discrepancy between the human's view and the embedding-level similarity, as our attack explicitly exploits the embedding space's geometry.

\begin{figure}[h]
  \centering
  \vspace{-15pt}
  \includegraphics[width=.93\linewidth]{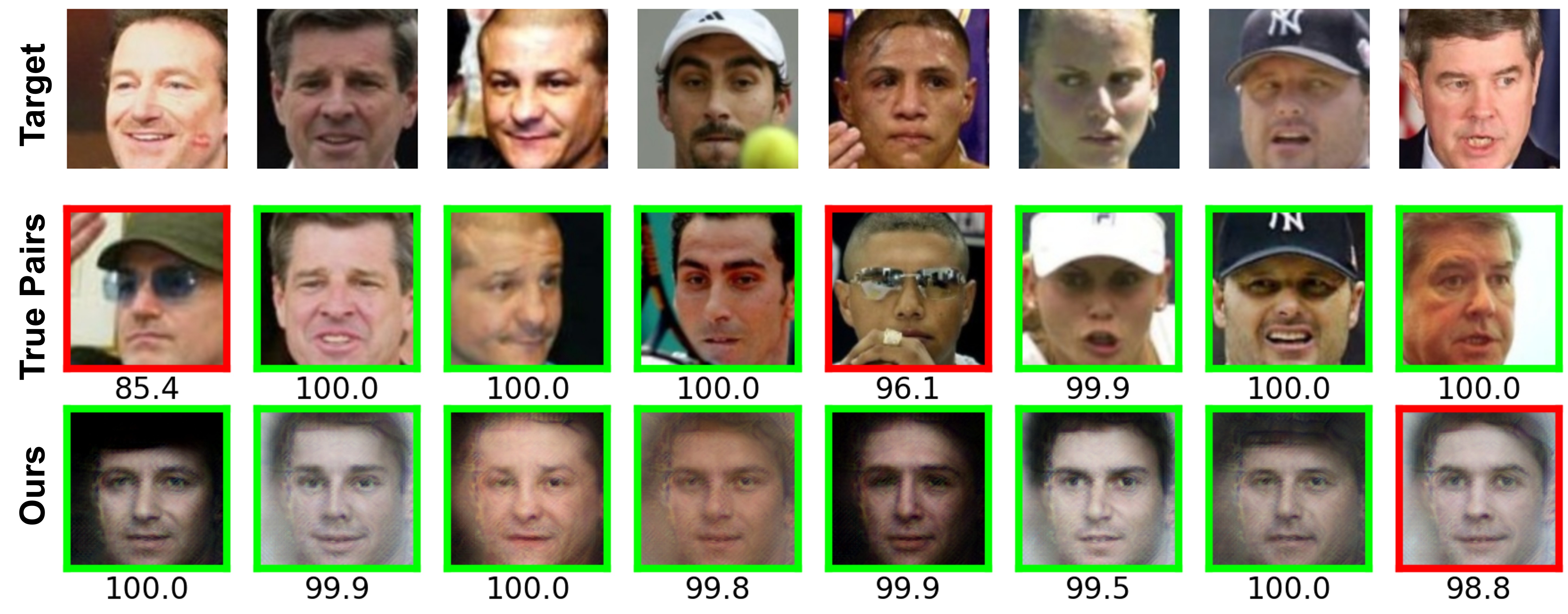}
  \vspace{-10pt}
  \caption{
  True pairs of the LFW dataset (1st and 2nd rows) and the reconstructed faces from our attack (3rd row) targeting those in the 1st row. Green/red boxes denote whether the confidence score with the 1st row's faces exceeds 99 or not, respectively.
  }
  \label{fig:aws_att_img}
  \vspace{-25pt}
\end{figure}



\subsubsection{Comparison with Additional Baselines}
We further examine whether three score-based attacks~\cite{razzhigaev2021darker,razzhigaev2025inverting,park2023towards} remain competitive in the 100-query regime. 
For~\cite{razzhigaev2021darker}, we use the query-wise trend reported in~\cite{kim2024scores}; for the other two attacks, we adapt their optimization procedures to the 100-query setting.
As these attacks rely on many iterative score queries, ranging from 4,000 to 300,000, their ASRs remain limited under this budget, reaching only about 1\% even at FMR$=10^{-4}$. In contrast, our attack achieves non-trivial ASRs under the same query budget and high-security thresholds. Detailed experimental settings and results are provided in Section~E.5 of the supplementary material.




\subsection{Ablation Study}

We conduct ablation studies on our techniques against $F_{\mathsf{A}}$; the results for $F_{1}$--$F_{3}$ are provided in Section~E.3 of the supplementary material.

\subsubsection{Query Budget Analysis}
Fig.~\ref{fig:asr_lfw_queries} shows ASR on LFW as the number of OFS queries increases from 10 to 100 under various decision thresholds. 
At thresholds 80 and 90, ASR increases rapidly and approaches 100\% with 60 queries. 
On the other hand, at the high-security setting (99), ASR remains low for small budgets, rising sharply after approximately 50 queries to achieve non-trivial ASR. 
Overall, 50 queries suffice to achieve a non-trivial ASR in all threshold regimes we tested.

\begin{figure}[!t]
    \centering
    \begin{subfigure}[t]{0.49\linewidth}
        \centering
        \includegraphics[width=\linewidth]{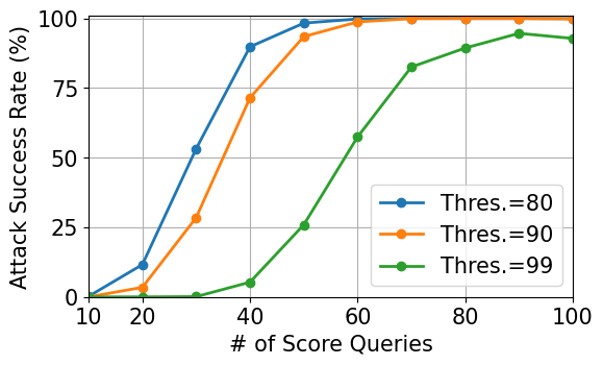}
        \caption{Query Budget Analysis}
        \label{fig:asr_lfw_queries}
    \end{subfigure}
    \hfill
    \begin{subfigure}[t]{0.49\linewidth}
        \centering
        \includegraphics[width=\linewidth]{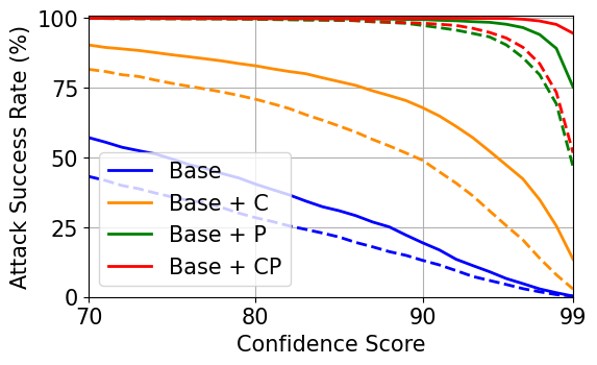}
        \caption{Effect of Techniques}
        \label{fig:th_sweep_type1}
    \end{subfigure}
    \vspace{-5pt}
    \caption{Ablation studies against $F_{\mathsf{A}}$ in the LFW dataset with respect to the effect of query budget and the effect of the correction matrix (C) and PCA-based OFS (P). 
    In (b), dashed lines and solid lines denote the ASR results from the baseline NbNet in~\cite{kim2024scores} and the proposed SPNet. Best viewed in color.}
    \label{fig:asr_lfw_threshold_sweep}
    \vspace{-10pt}
\end{figure}


\subsubsection{Effect of Proposed Techniques on ASR}
Fig.~\ref{fig:th_sweep_type1} depicts the effect of the proposed techniques to handle errors (\textbf{E1}-\textbf{E3}) on the ASR for $F_{\mathsf{A}}$.
We can observe that each technique (correction matrix, PCA-based OFS, and SPNet)  consistently and independently improves the ASR, even for both the baseline NbNet and the proposed SPNet.
This demonstrates how our techniques can overcome limitations of~\cite{kim2024scores} observed in Section~\ref{sec:investigateerrors} in a strict threshold setting.



\begin{wraptable}{R}{.55\linewidth}
\centering
\vspace{-10pt}
\caption{Ablation study on the inverse model}
{\scriptsize
\setlength{\tabcolsep}{4pt}
\begin{tabular}{lcc}
\toprule
Model  & Inv. Prec. & ASR on $F_{\mathsf{A}}$   \\ \midrule
NbNet~\cite{paik2025reversibility}     &  0.8809 & 56.10 \\ 
+ Tuning/Micro     &  0.9687 & 91.41 $\pm$ 1.75 \\ 
+ DSCasConv~\cite{shahreza2024vulnerability} & \textbf{0.9782} & \textbf{91.88 $\pm$ 1.60} \\ \bottomrule
\end{tabular} 
}
\label{tab:inv_ablation}
\vspace{-18pt}
\end{wraptable}


\subsubsection{Inverse Model}
We analyze the impact of our modifications to design SPNet. We train the baseline NbNet of~\cite{paik2025reversibility} for $F_{\mathsf{S}}$, and then add micro-level modifications and DSCasConv~\cite{shahreza2024vulnerability}, fixing the loss and the training recipe. Measuring the inversion precision on LFW and the ASR of our attack, Tab.~\ref{tab:inv_ablation} shows that our design improves both.
We further assess the robustness of both inversion precision and ASR to training randomness over five random seeds; the ASR on $F_{\mathsf{A}}$ exceeds $90\%$ across all five seeds, while the inversion precision remains nearly unchanged, varying by less than $10^{-4}$.

\section{Discussion}

\subsubsection{Potential Mitigation of the Attack}
To defend against the proposed attack, a straightforward solution is to hide the confidence score.
However, this may not be the fundamental solution because of insider adversaries who can access the raw API responses, e.g., service or API developers.
Instead, one may consider designing an FRS model having a stronger metric distortion---amplifying the error (\textbf{E1})---with respect to the open-source FRSs, e.g., FRSs that maintain high accuracy under tighter thresholds suggested by~\cite{kim2025non}.
To reduce the gain from the PCA, one may design a loss function that facilitates spreading the template distribution, such as some loss functions studied in the literature~\cite{duan2019uniformface}. 


We test our attack against ResNet-100 models with these defense methods applied.
With the fixed surrogate $F_\mathsf{S}$, our attack achieves ASRs of 41.33\% for~\cite{duan2019uniformface} at FMR$=10^{-6}$, which is roughly 47.94\% and 43.60\% drops compared to the same ResNet-100 target models $F_{2}$ and $F_{3}$.
In Section~E.1 of the supplementary material, we further show that~\cite{duan2019uniformface} yields a less concentrated PCA singular-value spectrum, thus increasing the error (\textbf{E2}).
In contrast, the ASR for \cite{kim2025non} is 0.13\% at FMR$=10^{-6}$, substantially weakening the attack efficacy.
We believe that these empirical results may help better understand the inherent robustness of FRS, leaving their further investigation as an interesting future direction.






\subsubsection{Image Quality of the Reconstructed Faces}
The proposed inverse model crafts relatively low-resolution faces compared to recent inversion models; in fact, our inverse model produces $128 \times 128$ resolution images, while recent models can handle $512 \times 512$ or $1024 \times 1024$ resolution~\cite{otroshi2023face,shahreza2025face,papantoniou2024arc2face}.
Nevertheless, all the reconstructed faces pass the face detection algorithm of the AWS CompareFace API without any additional treatments, and as we already observed in Tab.~\ref{tab:AWS_impersonation}, even reducing (\textbf{E1}) and (\textbf{E2}) alone already suffices to achieve a non-trivial ASR.
When we combine the proposed attack with Arc2Face~\cite{papantoniou2024arc2face}, which produces faces of $512 \times 512$ resolution, our attack still maintains 53.92\% of ASR when the threshold is set to 99, i.e., a high-security setting\footnote{We provide the example recovered faces in Section~E.4 of the supplementary material.}.
Improving high-resolution inverse models would further enhance the ASR.

\subsubsection{On the Feasibility of Real-World Physical Attacks}
While we showed the validity of the proposed attack against the commercial API, our evaluation considers a digital setting by directly submitting faces to the API.
However, in the real-world physical attack, the reconstructed faces may deform during presentation, e.g., printing or occlusion, and sensor-level defenses such as presentation attack detection~\cite{wang2020cross,wang2022patchnet,fang2023surveillance} can be employed.
Nevertheless, we emphasize that our attack exposes the vulnerability in the feature extraction stage, which is orthogonal to the sensor-level countermeasures mentioned above.
This suggests that the attack could potentially be realized in physical settings when combined with off-the-shelf spoofing attacks, leaving it as an interesting future direction.

\subsubsection{The Effect of Surrogate (Mis-)Alignment}
%
Surrogate-target misalignment (\textbf{E1}) plays a significant role in attack efficiency, as we observed in our analyses against $F_{1}$--$F_{3}$ and the defense method by~\cite{kim2025non}.
Since our primary goal is to analyze a commercial FRS $F_{\mathsf{A}}$, we do not exhaustively explore this aspect.
Nevertheless, its systematic characterization across diverse surrogate-target pairs in terms of architecture, training datasets, and losses would help us better understand the FRS vulnerability, and we leave it as important future work.




\section{Conclusion}


Despite the increasing adoption of FRSs in high-stakes scenarios, the integrity of FRSs under such scenarios has not yet been thoroughly explored.
In this paper, we demonstrate the vulnerability of various open-source and commercial FRSs even in low-FMR (e.g., $10^{-6}$) settings under black-box score-based attack scenarios with a practical query budget.
Our attack is enabled by a geometric interpretation of score queries over the embedding space, together with a novel inverse model that achieves high inversion precision. 
Beyond the attack itself, the proposed inverse model may be utilized for other applications, such as identity-preserving synthetic face generation~\cite{papantoniou2024arc2face,wu2025vec2face}.
We hope that our findings contribute to improving the security and robustness of future FRS designs.

\subsubsection{Ethical Statement}
Our goal is to understand the potential vulnerability of the FRSs against score-based attacks in strict security settings, not to facilitate misuse.
All the experiments were done in a strictly controlled setting, and \textbf{we did not launch attacks against real-world FRS deployments}.



\section*{Acknowledgements}
This work was supported in part by the Institute of Information and Communication Technology Planning and Evaluation (IITP) grant funded by the Korea Government (MSIT) (RS-2021-II210727), and in part by the Culture, Sports and Tourism R\&D Program through the Korea Creative Content Agency (KOCCA) grant funded by the Ministry of Culture, Sports and Tourism (RS-2024-00332210).

\clearpage  
%
%
\bibliographystyle{splncs04}
\bibliography{main}

\clearpage
\appendix
\title{Supplementary Material of \\ Breaking High Confidence: Practical Face Impersonation under High-Security Thresholds}

\titlerunning{Practical Face Impersonation under High-Security Thresholds}

\author{\ }
\institute{\ }
\authorrunning{C.~Kim et al.}


\maketitle












%
%

\appendix

\vspace{-1.0cm}


In this supplementary material, we provide the following content that was omitted in the main text due to space constraints.
\begin{itemize}
    \item The complete attack pipeline and omitted algorithms. (Section~\ref{app:SecA})

    \item Omitted mathematical proofs. (Section~\ref{app:SecB})

    \item Detailed configuration of the SPNet. (Section~\ref{app:SecC})    

    \item Specification of the train/evaluation datasets and FR models. (Section~\ref{app:SecD})

    \item Additional experimental results. (Section~\ref{app:SecE})

\end{itemize}

\section{Complete Attack Pipeline and Omitted Algorithms}\label{app:SecA}

In this section, we present our complete attack pipeline against the AWS CompareFace and provide detailed algorithms and analyses---converting confidence scores to cosine similarities and query cost for each attack---that were omitted in the main text.
Fig.~\ref{fig:spattack} presents a conceptual overview of the score-based impersonation attack pipeline used in our experiments. 

\begin{figure}[t]
    \centering
    \includegraphics[width=\linewidth]{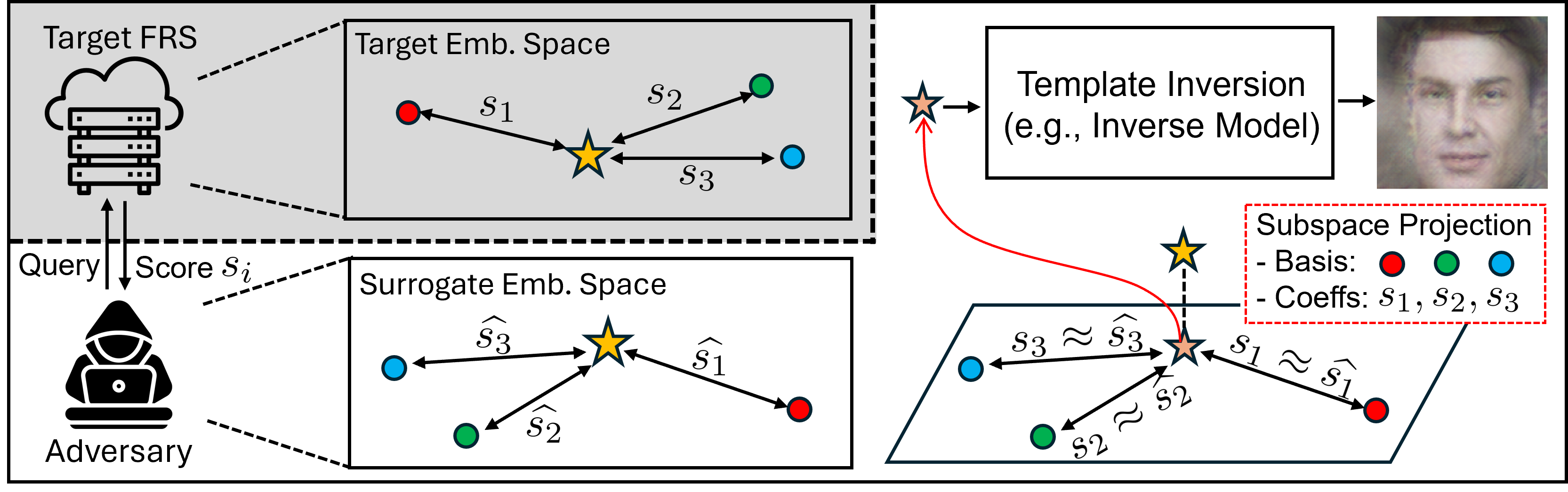}
    \vspace{-20pt}
    \caption{The score-based attack pipeline we used. 
    The gray area is not accessible to the adversary.
    A star and colored dots denote templates of the enrolled and queried faces in the target and surrogate embedding spaces, respectively. Best viewed in color.}
    \label{fig:spattack}
    \vspace{-10pt}
\end{figure}

\subsubsection{Commercial FRS with Score Queries}
We begin by providing a formal notation on the commercial FRS that permits score queries.
Let $\mathsf{FRS}$ denote a black-box FRS consisting of a feature extractor and an enrolled database. Given two face images, $\mathsf{FRS}$ computes the cosine similarity between their feature representations and returns a confidence score through an internal score-mapping function. 
For notational convenience, we write $\mathsf{FRS}(\mathsf{img}_1,~\mathsf{img}_2)$ for the confidence score returned for an input pair $(\mathsf{img}_1,~\mathsf{img}_2)\in\mathcal{I}\times\mathcal{I}$, where $\mathcal{I}$ denotes the set of face images. Let $F_T:\mathcal{I}\rightarrow\mathbb{S}^{d-1}$ denote the feature extractor used by $\mathsf{FRS}$, where $d$ is the feature dimension. Then,
\begin{align*}
    \mathsf{FRS}(\mathsf{img}_1,~\mathsf{img}_2)
    = g\!\left(\left\langle F_T(\mathsf{img}_1),~\,F_T(\mathsf{img}_2)\right\rangle\right)
    = c \in [0,~1],
\end{align*}
where $g:[-1,~1]\rightarrow[0,~1]$ maps cosine similarity to a confidence score. We assume that the adversary can query confidence scores not only against the enrolled target image but also for other image pairs. 
In particular, for an enrolled target image $\mathsf{img}_t$, we define
\begin{align*}
    \mathsf{FRS}_t(\mathsf{img}_q)
    = \mathsf{FRS}(\mathsf{img}_q,~\mathsf{img}_t)
    = c_q.
\end{align*}
This means the direct score query to the target face enrolled in $\mathsf{FRS}$.



\subsection{Full Attack Pipeline}
The proposed attack consists of two phases: the \textit{offline} preprocessing phase and the \textit{online} impersonation phase. 
The former can be done in advance without relying on the target face.
In the following paragraphs, we provide a detailed process for each phase.
The full attack pipeline is provided in Fig.~\ref{fig:pipeline}.

\subsubsection{Preprocessing Phase}
We first introduce the preprocessing phase.
The goal of this phase is to construct the PCA-based OFS and the correction matrix through pairwise score queries.
As these objects do not rely on the target face's identity, we regard this phase as a preprocessing.
\begin{itemize}
    \item [1.] First, prepare a surrogate feature extractor $F_L$ for the target FRS (e.g., by training one via metric learning or adapting an open-source model), and then train the corresponding inverse model $F_L^{-1}:\mathbb{S}^{d-1}\rightarrow\mathcal{I}$. The training procedure is described in the main paper.

    \item [2.] For a dataset $DB_L$, let $\mathcal{X}=F_L(DB_L)$ be the extracted feature set. After applying PCA to $\mathcal{X}$, let $\mathcal{O}\in\mathbb{R}^{k\times d}$ denote the matrix whose rows are the top-$k$ principal component vectors, where the $i$-th row $o_i\in\mathbb{R}^{d}$ represents the $i$-th principal direction in the feature space. The OFS is then generated by inverting these components, i.e.,
    $(\mathsf{O}_1,\ldots,\mathsf{O}_k)=\bigl(F_L^{-1}(o_1),\ldots,F_L^{-1}(o_k)\bigr)$. Let $\mathcal{O}_f\in\mathbb{R}^{k \times d}$ denote the OFS feature matrix, whose $i$-th row is $F_L(\mathsf{O}_i)$.

    \item [3.] Query $\mathsf{FRS}$ on all OFS pairs and let $R^c\in[0,~1]^{k\times k}$ denote the resulting confidence score matrix, whose $(i,j)$-th entry is $r^c_{ij}=\mathsf{FRS}(\mathsf{O}_i,~\mathsf{O}_j)$. The correction matrix $R=g^{-1}(R^c)\in[-1,~1]^{k\times k}$ is then computed, whose $(i,j)$-th entry is $g^{-1}(r^c_{ij})$. Here, $g^{-1}$ denotes an estimated inverse mapping from confidence scores to cosine similarities, which is described in Section~\ref{subsec:estimate_cos}.
\end{itemize}

\subsubsection{Impersonation Phase}
After obtaining the correction matrix and the OFS, the adversary now recovers the face by leveraging the direct score queries.
\begin{itemize}
    \item [1.] For each $\mathsf{O}_{i}$ ($i=1,\dots, k$), the adversary first makes a direct score query to obtain $c_{i} \gets \mathsf{FRS}_{t}$. Then it computes the estimated cosine similarity vector via obtained confidence scores, namely, $\vec{s} \gets (g^{-1}(c_{1}), \dots, g^{-1}(c_{k}))$.

    \item [2.] With the correction matrix $R$ and the OFS feature matrix $\mathcal{O}_{f}$ computed in the preprocessing phase, the adversary computes the candidate template $z \gets  \vec{s} \cdot R^{-1} \cdot \mathcal{O}_{f}$ in the local embedding space. 
    Finally, the adversary crafts the attack face image $\widehat{\mathsf{img}} \gets F_{L}^{-1}(z / \|z\|_{2})$ through the inverse model $F_{L}^{-1}$.
\end{itemize}




\begin{figure}[t]
\fbox{
\parbox{.95\linewidth}{

\begin{center}
    \textbf{Score-based Impersonation Attack against AWS CompareFace}
\end{center}
\vspace{-5pt}

\textbf{Inputs}:
\vspace{-5pt}
\begin{itemize}
    \item $\mathsf{FRS}$: black-box target face recognition system.
    \item $F_L:\mathcal{I}\rightarrow \mathbb{S}^{d-1}$: surrogate feature extractor.
    \item $DB_L$: reference database available to the attacker.
    \item $g^{-1}:[0,1]\rightarrow[-1,1]$: score inversion function.
    \item $k\in\mathbb{N}$: number of queries to the target system.
\end{itemize}
\vspace{-5pt}
\textbf{Output}: image $\widehat{\mathsf{img}}$.
\vspace{-5pt}

\rule{\linewidth}{0.1mm}

\textbf{Step 1. Preprocessing}
\vspace{-5pt}
\begin{enumerate}
    
    \item Train the inverse model $F_L^{-1}$ corresponding to $F_L$.
    \item Generate the orthogonal face set (OFS).
    \begin{enumerate}
        \item Extract feature vectors $\mathcal{X}\leftarrow F_L(DB_L)$.
        \item Compute the top-$k$ principal directions $(o_1,\ldots,o_k)\leftarrow \mathrm{PCA}(\mathcal{X},~k)$.
        \item Generate the OFS $(\mathsf{O}_1,\ldots,\mathsf{O}_k)\leftarrow      \bigl(F_L^{-1}(o_1),\ldots,F_L^{-1}(o_k)\bigr)$.
        \item Extract OFS features $\mathcal{O}_f \leftarrow
        \bigl(F_L(\mathsf{O}_1),\ldots,F_L(\mathsf{O}_k)\bigr)$.
    \end{enumerate}
    \item Compute the correction matrix $R$.
    \begin{enumerate}
        \item Query all OFS pairs:
$R^c=(r^c_{ij}),\ r^c_{ij}\leftarrow \mathsf{FRS}(\mathsf{O}_i,\mathsf{O}_j)$.
        \item Compute the correction matrix: $R \leftarrow g^{-1}(R^c)\in[-1,1]^{k\times k}$.
        \vspace{-10pt}
    \end{enumerate}
\end{enumerate}

\rule{\linewidth}{0.1mm}

\textbf{Step 2. Impersonation Attack against AWS}
\vspace{-5pt}
\begin{enumerate}
    \item Set $c_i \leftarrow \mathsf{FRS}_t(\mathsf{O}_i)$ via score queries with each OFS face for $i=1,\ldots,k.$.
    
    \item Compute the estimated cosine similarities: $\vec{s}\leftarrow (g^{-1}(c_i))_{i=1}^{k}$.

    \item Compute the candidate template: $z \leftarrow \vec{s}\cdot R^{-1}\cdot \mathcal{O}_f$.

    \item Generate the final attack image: $\widehat{\mathsf{img}} \leftarrow F_L^{-1}(z/\|z\|_{2})$.

    \item \textbf{Return} $\widehat{\mathsf{img}}$.    
    \vspace{-5pt}
\end{enumerate}
}
}
\vspace{-5pt}
\caption{Full pipeline of the proposed attack against AWS CompareFace.}
\label{fig:pipeline}

\vspace{-10pt}
\end{figure}

\subsection{Estimating Cosine Similarity from Confidence Scores}\label{subsec:estimate_cos}
%

In practice, we approximate $g^{-1}$ using the confidence score to the cosine similarity inversion proposed in~\cite{kim2024scores}. Rather than refitting the inversion parameters, we directly adopt the values provided in their publicly released implementation; these values were obtained on the CFP-FP~\cite{sengupta2016frontal} using the DOGBOX algorithm. 
This is because the surrogate model in~\cite{kim2024scores} and ours are nearly identical, including the architecture type, training dataset, training loss, and the training recipe, such as hyperparameters, except for the number of layers.
We empirically verify that keeping the same coefficients remains successful in the attack; we expect that improving this component would further increase the attack success rate.

In the case of AWS, the system returns a confidence score in $[0,100]$. We normalize this score to $[0,1]$, and let $c\in[0,1]$ denote the normalized confidence score. We then estimate cosine similarity via
\begin{align*}
    g^{-1}(c)
    = 1 - \left(
    -\frac{1}{k}\log\left(\frac{L}{c-b}-1\right)+d_0
    \right).
\end{align*}
We use the following parameters: $(L,d_0,k,b)=(0.99,\,0.75,\,-25,\,0.00004)$.

\subsection{Query Cost}
Throughout the attack, the adversary makes score queries when (i) constructing the correction matrix during preprocessing, and (ii) obtaining confidence scores between the target face and faces in the OFS.
When we select the size of the OFS as $k$, the former and latter take $\frac{k(k-1)}{2}$ and $k$ queries, respectively.
Since constructing the correction matrix can be performed \textit{offline}, i.e., it can be done independently of the target face, in the \textit{online} phase, $k$ queries suffice for conducting the impersonation attack.

In particular, the adversary can reduce the AWS API usage fee by using the method \texttt{search\_faces\_by\_image}, which compares a query image against all faces in an enrolled collection, rather than the one-to-one \texttt{compare\_faces} API. 
Note that \texttt{search\_faces\_by\_image} returns top 4,096 confidence scores in a single call, and each call requires \$0.001 of fee, i.e., the same fee as \texttt{compare\_faces}\footnote{\url{https://aws.amazon.com/rekognition/pricing/}}.
Using this, we construct the correction matrix by making $k=100$ calls to \texttt{search\_faces\_by\_image} and obtain 100 confidence scores from 100 calls to \texttt{compare\_faces}.
Ignoring the cost of enrolling the faces to the AWS cloud server, \$0.2 suffices to launch a single attack; furthermore, since constructing the correction matrix is one-time, subsequent attacks require \$0.1 API usage fee only.

\section{Omitted Proof}\label{app:SecB}
We present a detailed analysis of the projection error (\textbf{E2}), which was omitted in the main text.
We prove the following statement, which justifies our interpretation of the correction matrix from the perspective of metric distortion. 

\begin{proposition}\label{prop1}
Let $A \in \mathbb{R}^{k \times d}$ be a matrix and $x \in \mathbb{R}^{d}$ be a vector.
For a positive definite matrix $\Sigma \in \mathbb{R}^{d \times d}$, if we denote $s := A\Sigma x$ and $S = A\Sigma A^{T} \in \mathbb{R}^{k \times k}$, then $\widehat{x}:= A^{T}S^{-1}s$ satisfies the following:
\begin{align}\label{eq:prop1_1}
    \widehat{x} = \arg\min _{z \in R(A)} (x-z)^{T} \Sigma (x-z)
\end{align}
where $R(A)$ denotes the row space of $A$, i.e., the vector space that is spanned by row vectors of $A$.
\end{proposition}

\begin{proof}
Since $z \in R(A)$, by setting $z = A^{T}t$ for $t \in \mathbb{R}^{k}$, we can rewrite Eq.~\eqref{eq:prop1_1} as
\begin{align}
    t^{\ast} = \arg\min _{t \in \mathbb{R}^{k}} (x-A^{T}t)^{T} \Sigma (x-A^{T}t) \; \land \; \widehat{x} = At^{\ast} \nonumber.
\end{align}
Due to the symmetry of $\Sigma$, $z^{T} \Sigma x = x^{T} \Sigma z$, and we can observe that 
\begin{align}
    (x-z)^{T} \Sigma (x-z) &= x^{T} \Sigma x + z^{T}\Sigma z - 2z^{T} \Sigma x \nonumber \\
    &= x^{T}\Sigma x + t^{T}(A\Sigma A^{T})t - 2t^{T}(A\Sigma)x \nonumber.
\end{align}
Minimizing this quantity is a well-known (unconstrained) quadratic program, whose solution $t^{\ast}$ satisfies $(A \Sigma A^{T})t^{\ast} = A\Sigma x$. 
Thus, since $S=A \Sigma A^{T}$ and $s = A\Sigma x$, we finally obtain $t^{\ast} = S^{-1}A\Sigma x = S^{-1}s$ and $\widehat{x} = A^{T}t^{\ast} = A^{T}S^{-1}s$. \qed
\end{proof}

Notably, in the white-box attack scenario, $S=AA^{T}$, i.e., $\Sigma= I$.
In this case, we can observe that $A^{T}S^{-1}s = A^{T}(AA^{T})^{-1}s = A^{\dagger} s$, i.e., recovering the previous attack without the correction matrix.


\section{Detailed Configuration of the SPNet}\label{app:SecC}
As illustrated in Fig.~\ref{fig:spnet_overview}, SPNet\footnote{SPNet stands for Self-Perceptual Network, as it uses not only an external pretrained network but also the network itself to compute the perceptual loss.} consists of a Style Mapping Network and five Block Layers. The input feature is used both to produce a style feature through the Style Mapping Network and to initialize the reconstruction process for the Block Layers. Each Block Layer follows the DSCasConv~\cite{shahreza2024vulnerability} structure while incorporating AdaIN-based style modulation~\cite{huang2017arbitrary}.
%
\subsubsection{Style Mapping Network}
The Style Mapping Network maps the input feature $[B,~512]$ to a style feature of the same shape, where $B$ denotes the batch size and 512 the feature dimension. It consists of 8 FC blocks, each composed of $\mathsf{EqualLinear}$, $\mathsf{PixelNorm}$ (pixel-wise feature vector normalization), and $\mathsf{ReLU}$ (rectified linear unit). $\mathsf{EqualLinear}$ follows the equalized learning rate of Karras et al.~\cite{karras2017progressive}, and maps $[B,~512]$ to $[B,~512]$ in this network. The resulting style feature is applied to all $\mathsf{AdaIN}$ modules in the Block Layers.
\subsubsection{Block Layers}
Before entering the first Block Layer, the input feature is projected by an $\mathsf{FC}$ (fully connected) layer and $\mathsf{GELU}$ (gaussian error linear unit~\cite{hendrycks2016gaussian}), expanding $[B,~512]$ to $[B,~512\times4\times4]$ and reshaping it to $[B,~512,~4,~4]$. As highlighted by the red dashed box in Fig.~\ref{fig:spnet_overview}, the five Block Layers share the same overall structure while differing in the number of internal Style Blocks. Each Block Layer begins with $\mathsf{DeConv}$-$\mathsf{PixelNorm}$-$\mathsf{GELU}$, where $\mathsf{DeConv}$ (transposed convolution) upsamples the spatial resolution by a factor of 2, followed by $n_b$ residual Style Blocks of the form $\mathsf{Conv}$-$\mathsf{AdaIN}$-$\mathsf{GELU}$. After these $n_b$ residual Style Blocks, one additional Style Block is applied without a residual connection. Across the five Block Layers, the feature shape is progressively transformed as
\begin{align*}
[B,~512,~4,~4]
&\xrightarrow[\text{$n_b=32$}]{\text{Block Layer 1}}
[B,~512,~8,~8] \\
&\xrightarrow[\text{$n_b=16$}]{\text{Block Layer 2}}
[B,~256,~16,~16] \\
&\xrightarrow[\text{$n_b=8$}]{\text{Block Layer 3}}
[B,~128,~32,~32]\\
&\xrightarrow[\text{$n_b=4$}]{\text{Block Layer 4}}
[B,~64,~64,~64]\\
&\xrightarrow[\text{$n_b=2$}]{\text{Block Layer 5}}
[B,~32,~128,~128].
\end{align*}
After the five Block Layers, the final feature map is projected to the RGB image space by a $3\times3$ convolution layer with output shape $[B,~3,~128,~128]$.

\begin{figure}[t]
    \centering
    \includegraphics[width=\linewidth]{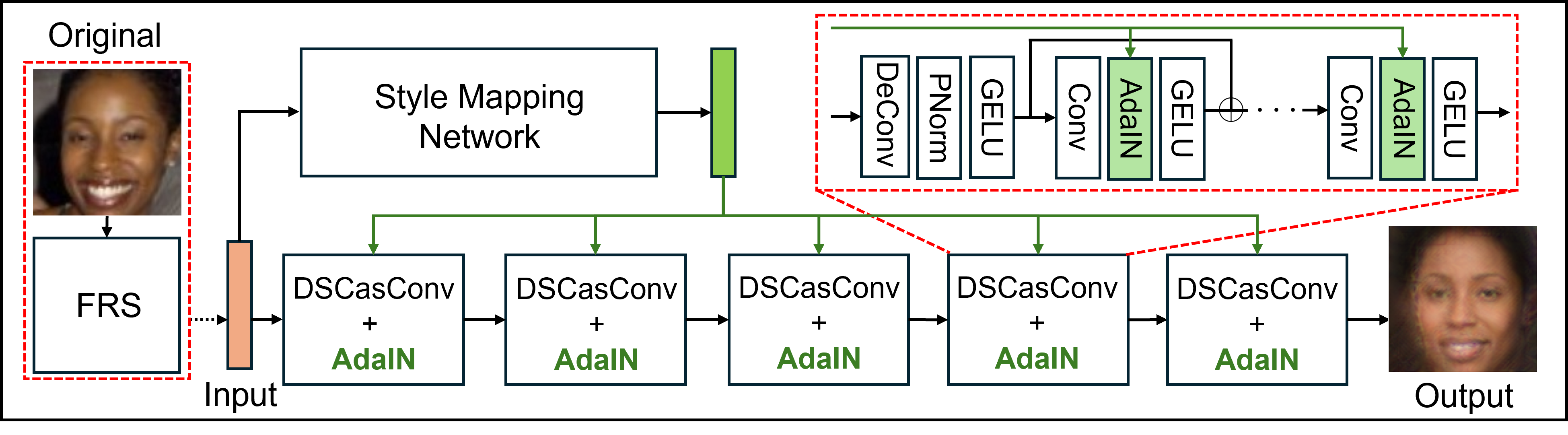}
    \vspace{-20pt}
    \caption{The design of SPNet. PNorm: Pixel-wise normalization. Best viewed in color.}
    \label{fig:spnet_overview}
    \vspace{-15pt}
\end{figure}

\section{Details on the Datasets \& FR Models}\label{app:SecD}
%

\subsubsection{Specification of the Datasets}
We provide the detailed specifications---e.g., the number of images, identities, and true/false pairs---of the datasets we used, including the evaluation datasets and the CASIA-WebFace dataset~\cite{yi2014learning}.
The details are provided in Tab.~\ref{tab:dataspec}.
We note that all the datasets used during experiments are publicly available.
In particular, for LFW~\cite{huang2008labeled}, CFP-FP~\cite{sengupta2016frontal}, and AgeDB~\cite{moschoglou2017agedb}, we used the subsets for face verification benchmarks instead of the full datasets, by following the convention in the literature.

\begin{table}[h]
    \centering
    \caption{Detailed specification of the datasets. TP/FP: True/false pairs.}    
    \vspace{-10pt}
    \setlength{\tabcolsep}{5pt}
    {\scriptsize
    \begin{tabular}{ccccc}
    \toprule
        \textbf{Datasets} & \textbf{LFW} & \textbf{CFP-FP} & \textbf{AgeDB} & \textbf{CASIA-WebFace}  \\ \midrule
        \# of Imgs & 13,233 & 7,000 & 16,488 & 490,623 \\
        \# of IDs & 5,749 & 500 & 568 & 10,572  \\
        \# of TP/FPs & 3,000/3,000 & 3,500/3,500 & 3,000/3,000 & NA  \\        
    \bottomrule    
    \end{tabular}
    }
    \label{tab:dataspec}
\end{table}

\subsubsection{Estimated Thresholds at Various FMR Levels}
We also provide the thresholds measured at various FMR levels, from $10^{-2}$ to $10^{-6}$, of the open-source FRSs ($F_{1}$--$F_{3}$) and the surrogate model ($F_{\mathsf{S}}$).
Recall that all the main experiments were focused on the $\mathrm{FMR}={10^{-6}}$.
The thresholds are provided in Tab.~\ref{tab:summary_frs_threshold}.

\begin{table}[t]
    \centering
    \caption{Estimated thresholds of open-source FRSs at various FMR levels}
    \vspace{-10pt}
    \setlength{\tabcolsep}{7pt}
    {\scriptsize
    \begin{tabular}{cccccc}
    \toprule
        \textbf{FMRs} & $\mathbf{10^{-2}}$ & $\mathbf{10^{-3}}$ & $\mathbf{10^{-4}}$ & $\mathbf{10^{-5}}$ & $\mathbf{10^{-6}}$  \\ \midrule
        $F_{\mathsf{S}}$         & 0.1416 & 0.1981 & 0.2526 & 0.3095 & 0.5422 \\ 
        $F_{1}$         & 0.1390 & 0.1932 & 0.2464 & 0.3010 & 0.5195 \\ 
        $F_{2}$         & 0.1615 & 0.2265 & 0.2881 & 0.3463 & 0.5216 \\ 
        $F_{3}$         & 0.1864 & 0.2646 & 0.3324 & 0.3942 & 0.5537 \\ 
    \bottomrule
    \end{tabular}
    }
    \label{tab:summary_frs_threshold}
    \vspace{-15pt}
\end{table}

\subsubsection{Benchmark Results of FRSs on Evaluation Datasets}
We report the true match rate (TMR) at various FMR levels. 
First, in Tab.~\ref{tab:bench_open_prot}, we provide the benchmark results of $F_{1}$--$F_{3}$ and $F_\mathsf{S}$ on the LFW, CFP-FP, and AgeDB datasets, reporting the TMR at the FMRs of $10^{-2}$ and $10^{-3}$.
In addition, to measure TMRs under tighter FMR settings, we use the estimated threshold reported in Tab.~\ref{tab:bench_open_estim}.
From these tables, we can observe that the thresholds measured in Tab.~\ref{tab:bench_open_prot} tend to be similar to or smaller than those estimated from the CASIA-WebFace dataset (i.e., Tab.~\ref{tab:summary_frs_threshold}).
This suggests that the estimated thresholds may be conservative, and the ASR may be inflated when using thresholds measured on benchmark datasets that follow the same distribution as LFW, CFP-FP, and AgeDB.

\begin{table}[h]
    \centering
    \vspace{-15pt}
    \caption{Benchmark results of open-source FRSs in evaluation datasets with thresholds. 
    We report TMR(\%) at $\mathrm{FMR}=10^{-2}$ and $10^{-3}$ based on the benchmark.
    $\tau$ denotes the threshold obtained from the benchmark.
    }
    \vspace{-10pt}
    \setlength{\tabcolsep}{4pt}
    \resizebox{\linewidth}{!}{
    \begin{tabular}{ccccccccccccc}
    \toprule
    FRSs & \multicolumn{4}{c}{LFW} & \multicolumn{4}{c}{CFP-FP} & \multicolumn{4}{c}{AgeDB} \\ \cmidrule(lr){1-1}\cmidrule(lr){2-5} \cmidrule(lr){6-9} \cmidrule(lr){10-13}
    FMRs    & $10^{-2}$ & $\tau$ & $10^{-3}$ & $\tau$ & $10^{-2}$ & $\tau$ & $10^{-3}$ & $\tau$ & $10^{-2}$ & $\tau$ & $10^{-3}$ & $\tau$  \\ \midrule
    $F_{\mathsf{S}}$ & 99.80 & 0.1446 & 99.70 & 0.1888 & 98.97 & 0.1414 & 98.74 & 0.1985 & 97.53 & 0.1786 & 97.50 & 0.1814\\
    $F_{1}$ & 99.77 & 0.1380 & 99.73 & 0.1670 & 98.89 & 0.1380 & 98.74 & 0.1837 & 97.10 & 0.1699 & 97.07 & 0.1771\\
    $F_{2}$ & 99.80 & 0.1581 & 99.73 & 0.2046 & 98.37 & 0.1594 & 98.17 & 0.1835 & 97.50 & 0.1914 & 97.10 & 0.2116\\
    $F_{3}$ & 99.77 & 0.1835 & 99.73 & 0.2500 & 98.37 & 0.1871 & 98.14 & 0.2073 & 97.50 & 0.2157 & 97.50 & 0.2161\\    
    \bottomrule
    \end{tabular}
    }
    \label{tab:bench_open_prot}
    \vspace{-15pt}
\end{table}

\begin{table}[h]
    \centering
    \caption{Benchmark results of open-source FRSs in evaluation datasets. We report TMR(\%) at the estimated thresholds at $\mathrm{FMR}=10^{-4}$, $10^{-5}$, and $10^{-6}$, respectively.}
    \vspace{-10pt}
    \setlength{\tabcolsep}{6pt}
    {\scriptsize
    \begin{tabular}{cccccccccc}
    \toprule
    FRSs & \multicolumn{3}{c}{LFW} & \multicolumn{3}{c}{CFP-FP} & \multicolumn{3}{c}{AgeDB} \\ \cmidrule(lr){1-1}\cmidrule(lr){2-4} \cmidrule(lr){5-7} \cmidrule(lr){8-10}
    FMRs    & $10^{-4}$ & $10^{-5}$ & $10^{-6}$    &  $10^{-4}$ & $10^{-5}$ & $10^{-6}$      &  $10^{-4}$ & $10^{-5}$ & $10^{-6}$      \\ \midrule
    $F_{\mathsf{S}}$ & 99.70 & 99.63 & 94.47 & 98.34 & 97.29 & 53.60 & 95.73 & 93.43 & 34.23
 \\
    $F_{1}$ & 99.67 & 99.60 & 95.43 & 98.00 & 96.34 & 53.14 & 95.03 & 91.83 & 32.03
 \\
    $F_{2}$ & 99.70 & 99.57 & 97.23 & 95.80 & 91.89 & 60.11 & 95.40 & 92.80 & 58.57
 \\
    $F_{3}$ & 99.70 & 99.53 & 96.53 & 94.06 & 90.03 & 53.54 & 94.50 & 90.83 & 51.03
 \\    
    \bottomrule
    \end{tabular}
    }
    \label{tab:bench_open_estim}
    \vspace{-15pt}
\end{table}


\section{Additional Experimental Results}\label{app:SecE}
%

\subsection{Attack Efficacy and Non-Uniformity of Templates}
Recall that, when attacking $F_{2}$ and $F_{3}$, we observe that the ASRs measured from them exceed that of the white-box attack, i.e., $F_{\mathsf{S}}$. 
We hypothesize that this phenomenon stems from the extent of non-uniformity of the embedding space; the embeddings are closer to the subspace spanned by the PCA-based OFSs compared to $F_{\mathsf{S}}$.
That is, the amount of reducing projection gap (\textbf{E2}) in $F_{2}$ and $F_{3}$ would be larger than that of $F_{\mathsf{S}}$.
Thus, if such a reduction dominates the error induced by the metric distortion (\textbf{E1}), then achieving a higher ASR than the white-box attack setting is possible.

\begin{figure}[t]
    \centering
    \includegraphics[width=\linewidth]{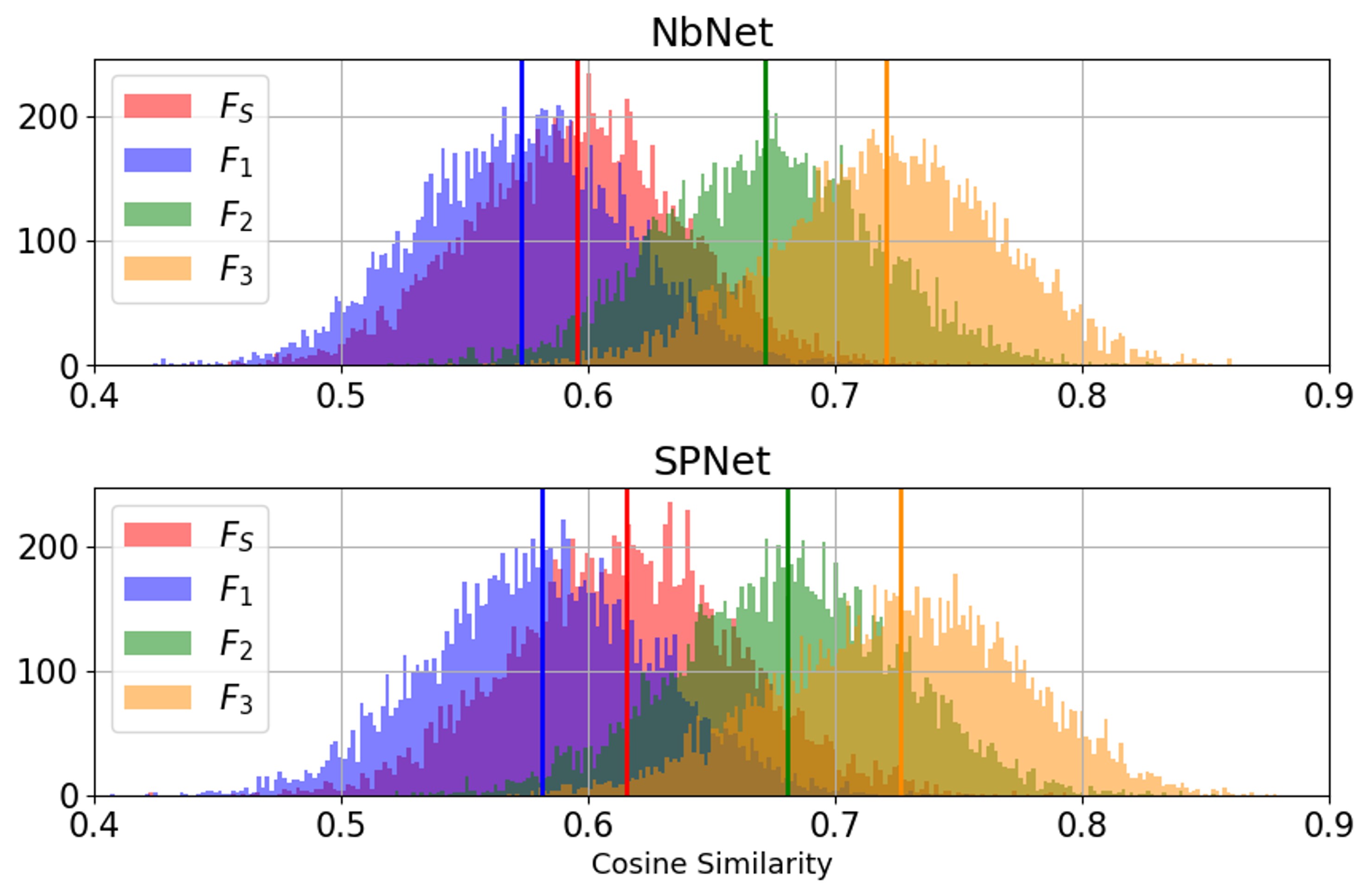}
    \caption{Projection gaps between the embeddings from the LFW dataset and the PCA-based OFS measured from various FRSs.
    The OFS is crafted from $F_{\mathsf{S}}$ with NbNet (upper) and SPNet (lower).
    Solid lines indicate the average of the cosine similarities of the projection gap. Best viewed in color.
    }
    \label{fig:proj_gap_test}
    \vspace{-10pt}
\end{figure}

To verify this hypothesis, we measure the projection gap on $F_{1}$--$F_{3}$ and $F_{\mathsf{S}}$ between the embeddings from the LFW dataset and the PCA-based OFS constructed from $F_{\mathsf{S}}$.
Following our main attack experiments, we craft the PCA-based OFS from the MS1MV3 dataset from the baseline NbNet and the proposed SPNet.
The visualization results are provided in Fig.~\ref{fig:proj_gap_test}.
We can observe that the projection gap on $F_{1}$ is larger (i.e., smaller cosine similarities) than $F_{\mathsf{S}}$, whereas the gaps on $F_{2}$ and $F_{3}$ are smaller (i.e., larger cosine similarities).
Such a tendency coincides with the observation that the ASRs on $F_{2}$ and $F_{3}$ in the black-box setting exceed that of $F_{\mathsf{S}}$ in the white-box setting, while this is not the case for $F_{1}$.

We further examine the source of these projection gap differences through the PCA singular-value spectrum. We conduct this spectral comparison within the same architecture, using $F_\mathsf{S},\ F_2,\ F_3$, and additionally include the distribution-spreading model~\cite{duan2019uniformface}, which is explicitly designed to spread the embedding spectrum. As shown in Fig.~\ref{fig:singular}, models with less concentrated singular-value spectra exhibit lower projection similarities onto the 100-dimensional OFS subspace, indicating a larger projection gap. Consistently, measurement of (\textbf{E2}) yields mean projection gaps of 0.6878, 0.7324, 0.7884, and 0.8288 for $F_3$, $F_2$, $F_{\mathsf{S}}$, and ~\cite{duan2019uniformface}, respectively. This confirms that distribution spreading weakens the PCA-based OFS by increasing the projection gap.

\begin{figure}[t]
    \centering
    \begin{subfigure}[t]{\linewidth}
        \includegraphics[width=\linewidth]{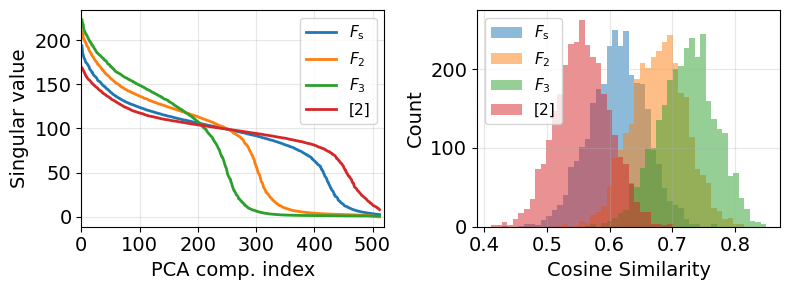}
        \vspace{-10pt}        
    \end{subfigure}
    \vfill
    \caption{The left plot shows the PCA singular-value spectra of MS1MV3~\cite{deng2020retinaface} embeddings. The right plot shows the distributions of cosine similarities between LFW embeddings and their projections onto the OFS subspace.}
    \label{fig:singular}
    \vspace{-10pt}
\end{figure}

\subsection{Additional ASR Results on Open-Source FRSs}
We provide additional comparisons on the ASR results against open-source FRSs when $\mathrm{FMR}=10^{-4}, 10^{-5}$, and $10^{-6}$. 
In particular, since attacking $F_{\mathsf{S}}$ is the white-box attack, we do not apply the correction matrix during comparison.
The results are provided in Tab.~\ref{tab:ablation_merged}.
We can observe that the proposed techniques progressively improve the ASR, except for the correction matrix; it even degrades the ASR, unlike the case of attacking the AWS CompareFace ($F_{\mathsf{A}}$).
We suspect that this is because the correction matrix does not reduce the metric distortion for these open-source FRSs.
Nevertheless, our techniques improve the ASR from 32.06\% to 74.43\% compared to the baseline, showing their effectiveness.

\begin{table*}[!t]
\centering
\caption{ASRs (\%) on benchmark datasets against $F_{\mathsf{S}}$ and $F_{1}$--$F_{3}$. Baseline corresponds to ``NbNet'' without our techniques. C: correction matrix. P: PCA-based OFS.}
\vspace{-10pt}
\setlength{\tabcolsep}{3pt}
\begin{subtable}{\linewidth}
\centering
{\scriptsize
\begin{tabular}{l ccc ccc ccc}
\toprule
\multicolumn{1}{c}{Used Tools/Tech.} & \multicolumn{3}{c}{LFW} & \multicolumn{3}{c}{CFP-FP} & \multicolumn{3}{c}{AgeDB} \\
\cmidrule(lr){2-4} \cmidrule(lr){5-7} \cmidrule(lr){8-10}
FMR & $10^{-4}$ & $10^{-5}$ & $10^{-6}$ & $10^{-4}$ & $10^{-5}$ & $10^{-6}$ & $10^{-4}$ & $10^{-5}$ & $10^{-6}$ \\
\midrule
Baseline~\cite{kim2024scores} & 99.80 & 96.77 & 16.77 & 99.80 & 96.91 & 14.54 & 99.87 & 98.17 & 18.73   \\
\midrule 
NbNet+P & 100 & 99.97 & 43.47 & 100 & 100 & 50.43 & 100 & 100 & 55.03   \\
Arc2Face+P & 99.97 & 99.23 & 15.53 & 99.91 & 99.49 & 17.51 & 100 & 99.57 & 16.00
 \\
\midrule 
SPNet+P & 100 & 100 & \textbf{76.83} & 100 & 100 & \textbf{81.63} & 100 & 100 & \textbf{87.03} \\
\bottomrule
\end{tabular}
}
\caption{$F_{\mathsf{S}}$}
\end{subtable}


\begin{subtable}{\linewidth}
\centering
{\scriptsize
\begin{tabular}{l ccc ccc ccc}
\toprule
\multicolumn{1}{c}{Used Tools/Tech.} & \multicolumn{3}{c}{LFW} & \multicolumn{3}{c}{CFP-FP} & \multicolumn{3}{c}{AgeDB} \\
\cmidrule(lr){2-4} \cmidrule(lr){5-7} \cmidrule(lr){8-10}
FMR & $10^{-4}$ & $10^{-5}$ & $10^{-6}$ & $10^{-4}$ & $10^{-5}$ & $10^{-6}$ & $10^{-4}$ & $10^{-5}$ & $10^{-6}$ \\
\midrule
Baseline & 88.93 & 63.97 & 0.57 & 89.77 & 65.83 & 0.37 & 93.57 & 70.53 & 1.00   \\
\midrule 
NbNet+C & 87.13 & 61.43 & 0.57  & 88.31 & 64.09 & 0.60 & 92.06 & 68.57 & 0.87   \\
NbNet+P & 99.86 & 96.23 & 10.37   & 99.78 & 98.00 & 15.66 & 99.97 & 98.53 & 19.43   \\
\midrule 
NbNet+CP             & 99.73 & 96.00 & 11.47 & 99.83 & 97.63 & 15.74 & 100 & 98.73 & 21.40 \\
Arc2Face+CP & 96.73 & 83.73 & 2.43 & 98.51 & 89.26 & 4.00 & 98.30 & 90.37 & 2.77 \\
\midrule 
SPNet+CP & 99.96 & 99.63 & \textbf{32.63} & 100 & 99.91 & \textbf{37.20} & 100 & 99.90 & \textbf{50.33} \\
\bottomrule
\end{tabular}
}
\caption{$F_{1}$}
\end{subtable}


\begin{subtable}{\linewidth}
\centering
{\scriptsize
\begin{tabular}{l ccc ccc ccc}
\toprule
\multicolumn{1}{c}{Used Tools/Tech.} & \multicolumn{3}{c}{LFW} & \multicolumn{3}{c}{CFP-FP} & \multicolumn{3}{c}{AgeDB} \\
\cmidrule(lr){2-4} \cmidrule(lr){5-7} \cmidrule(lr){8-10}
FMR & $10^{-4}$ & $10^{-5}$ & $10^{-6}$ & $10^{-4}$ & $10^{-5}$ & $10^{-6}$ & $10^{-4}$ & $10^{-5}$ & $10^{-6}$ \\
\midrule
Baseline & 98.83 & 91.57 & 24.13 & 98.83 & 91.09 & 18.66 & 99.10 & 93.47 & 19.97   \\
\midrule 
NbNet+C & 98.17 & 90.40 & 22.20 & 97.86 & 88.71 & 19.26 & 98.53 & 91.50 & 18.77   \\
NbNet+P & 100 & 99.93 & 69.13 & 100 & 99.94 & 74.66 & 100 & 100 & 75.10   \\
\midrule 
NbNet+CP             & 100 & 99.83 & 69.43 & 100 & 99.94 & 74.69 & 100 & 99.97 & 75.97 \\
Arc2Face+CP & 99.27 & 93.27 & 19.87 & 99.26 & 95.57 & 24.74 & 99.23 & 94.60 & 19.30 \\
\midrule 
SPNet+CP & 100 & 100 & \textbf{89.27} & 100 & 99.94 & \textbf{93.09} & 100 & 100 & \textbf{93.47} \\
\bottomrule
\end{tabular}
}
\caption{$F_{2}$}
\end{subtable}


\begin{subtable}{\linewidth}
\centering
{\scriptsize
\begin{tabular}{l ccc ccc ccc}
\toprule
\multicolumn{1}{c}{Used Tools/Tech.} & \multicolumn{3}{c}{LFW} & \multicolumn{3}{c}{CFP-FP} & \multicolumn{3}{c}{AgeDB} \\
\cmidrule(lr){2-4} \cmidrule(lr){5-7} \cmidrule(lr){8-10}
FMR & $10^{-4}$ & $10^{-5}$ & $10^{-6}$ & $10^{-4}$ & $10^{-5}$ & $10^{-6}$ & $10^{-4}$ & $10^{-5}$ & $10^{-6}$ \\
\midrule
Baseline & 97.13 & 86.47 & 20.50 & 97.09 & 83.83 & 17.31 & 96.87 & 84.20 & 16.87   \\
\midrule 
NbNet+C & 94.67 & 81.30 & 20.00 & 94.86 & 78.86 & 18.34 & 93.57 & 78.10 & 15.80   \\
NbNet+P & 99.97 & 99.67 & 66.57 & 100 & 99.94 & 69.49 & 100 & 100 & 69.17   \\
\midrule 
NbNet+CP             & 99.97 & 99.47 & 64.53 & 100 & 99.71 & 68.23 & 100 & 99.83 & 67.43 \\
Arc2Face+CP & 97.37 & 87.73 & 18.70 & 97.89 & 88.20 & 21.23 & 97.07 & 87.47 & 18.40 \\
\midrule 
SPNet+CP & 100 & 99.93 & \textbf{84.93} & 100 & 100 & \textbf{86.71} & 100 & 100 & \textbf{88.13} \\
\bottomrule
\end{tabular}
}
\caption{$F_{3}$}
\end{subtable}
\label{tab:ablation_merged}
\vspace{-20pt}
\end{table*}

In addition, we provide the visualization of the cosine similarities of true/false pairs, the CFP-FP and AgeDB datasets, with those from the proposed attack (e.g., see Fig.~3 of the main text).
The results are visualized in Fig.~\ref{fig:opensource_cfpage}.
In contrast to the results in the LFW dataset, the cosine similarity values from our attack tend to exceed those between true pairs.
This is because of the CFP-FP and AgeDB datasets' relatively larger intra-class variations compared to the LFW; the former two datasets consist of hard pairs of facial images, e.g., faces of the same identity but different pose or age.

\begin{figure}[t]
    \centering
    \begin{subfigure}[t]{\linewidth}
        \includegraphics[width=\linewidth]{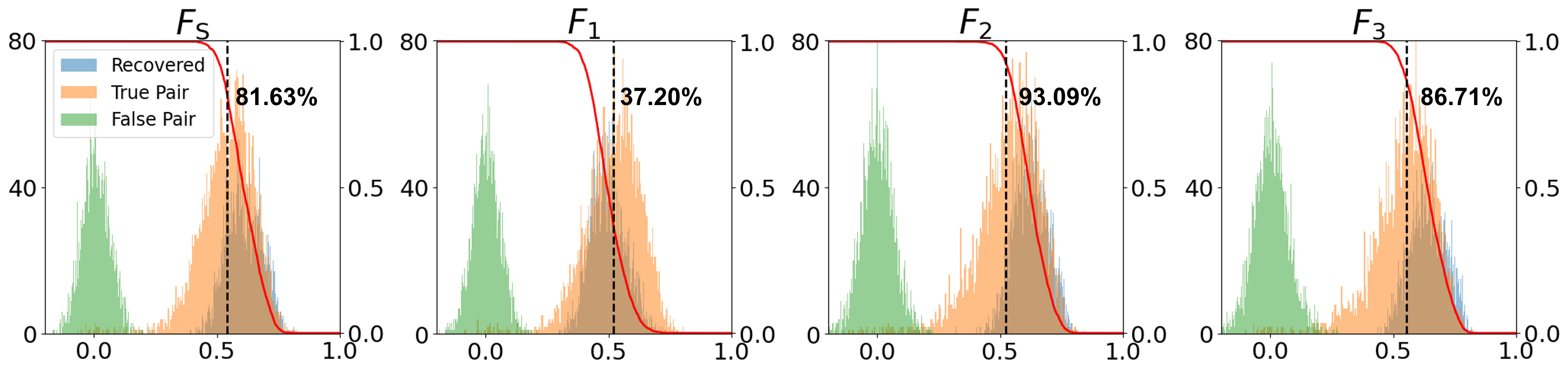}
        \vspace{-10pt}        
        \caption{CFP-FP}
    \end{subfigure}
    \vfill
    \begin{subfigure}[t]{\linewidth}
        \includegraphics[width=\linewidth]{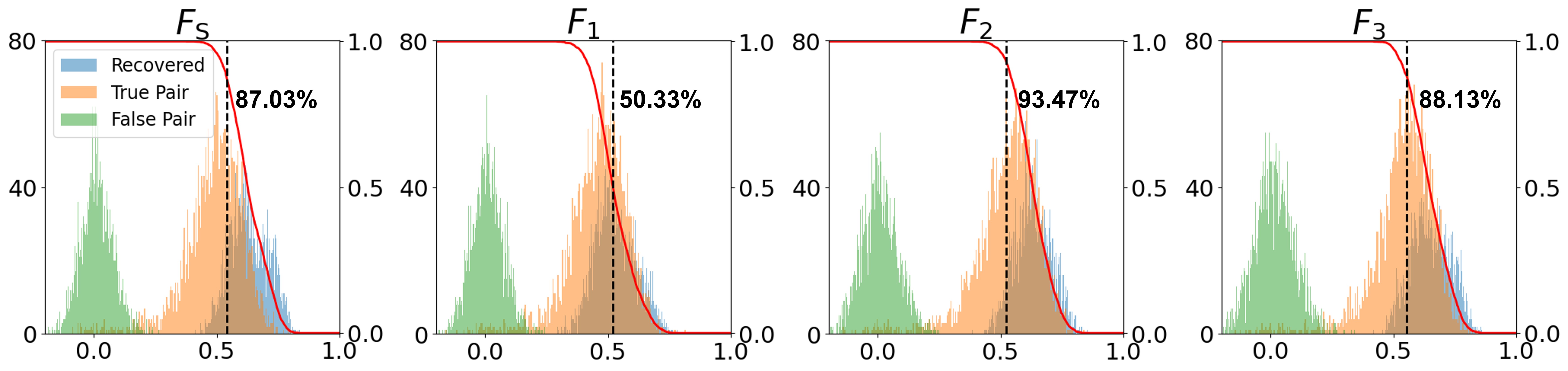}
        \vspace{-10pt}
        \caption{AgeDB}
    \end{subfigure}
    \vspace{-10pt}
    \caption{Attack results of ours on various open-source FRSs in the CFP-FP and AgeDB datasets. Best viewed in color.}
    \label{fig:opensource_cfpage}
    \vspace{-10pt}
\end{figure}

\subsection{Additional Ablation Studies on Open-Source FRSs}
As in Fig.~5 of the main text for $F_{\mathsf{A}}$, we provide ablation study results on open-source FRSs $F_{1}$--$F_{3}$ and $F_{\mathsf{S}}$.
We first provide the query sweep analysis result on various FMR levels, including $10^{-4}$, $10^{-5}$, and $10^{-6}$.
The results are visualized in Fig.~\ref{fig:abstudy_query}.
Similar to $F_{\mathsf{A}}$, we can observe that 50--60 queries are sufficient for achieving non-trivial ASRs across all attack scenarios and FMR levels.

In addition, we also provide the trajectories of the ASR with respect to the thresholds, progressively adding our techniques to reduce errors (\textbf{E1})--(\textbf{E3}).
The results are provided in Fig.~\ref{fig:abstudy_opensource}.
As we can expect from the result of Tab.~\ref{tab:ablation_merged}, the effect of the correction matrix is marginal across all the experimental settings, while the adoption of the PCA-based OFS and the proposed SPNet shows a sharp improvement on the ASR.

\begin{figure}
    \centering
    \begin{subfigure}[t]{\linewidth}
        \centering        
        \includegraphics[width=.95\linewidth]{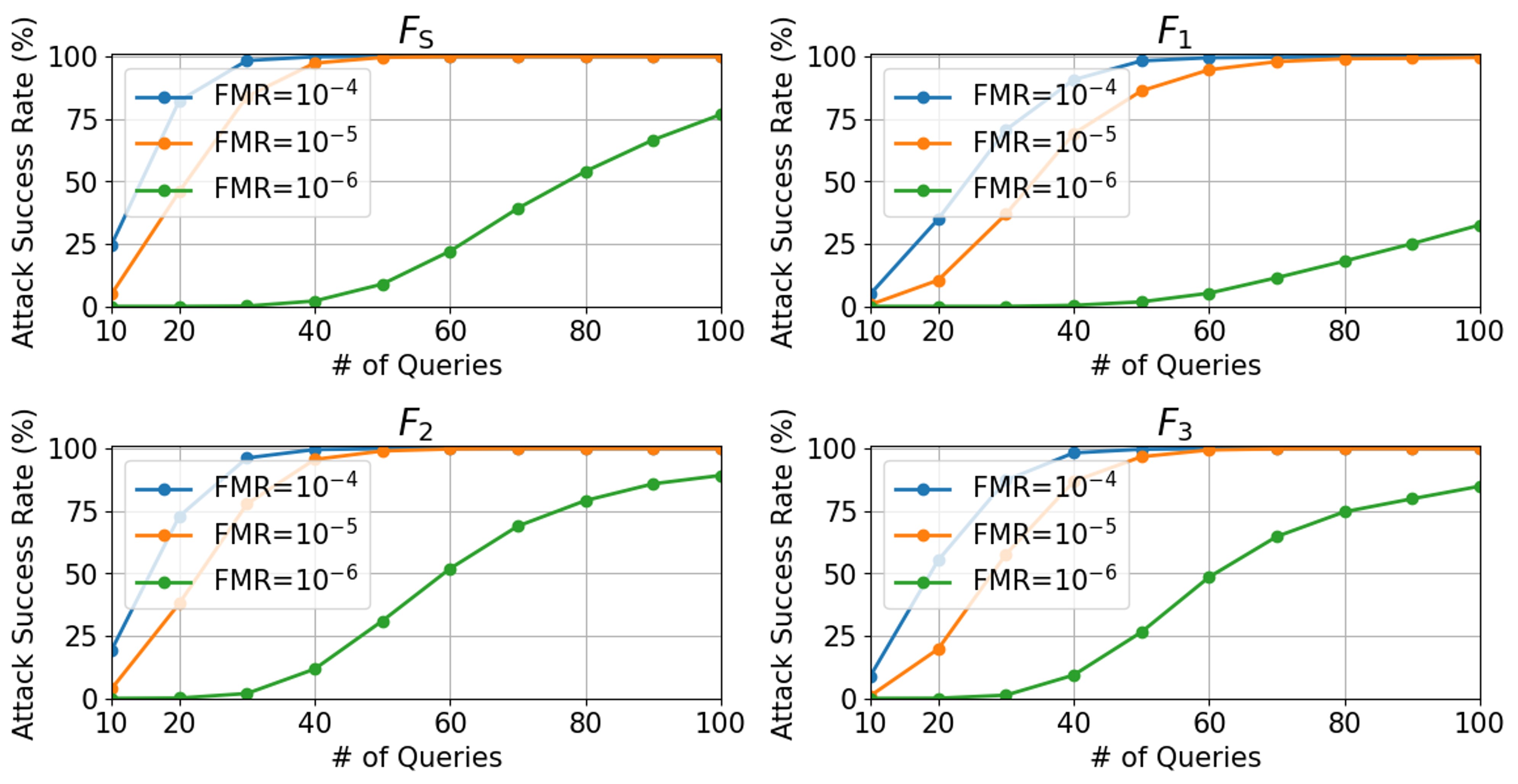}
        \vspace{-7pt}
        \caption{LFW}
    \end{subfigure}
    \vfill
    \begin{subfigure}[t]{\linewidth}
        \centering        
        \includegraphics[width=.95\linewidth]{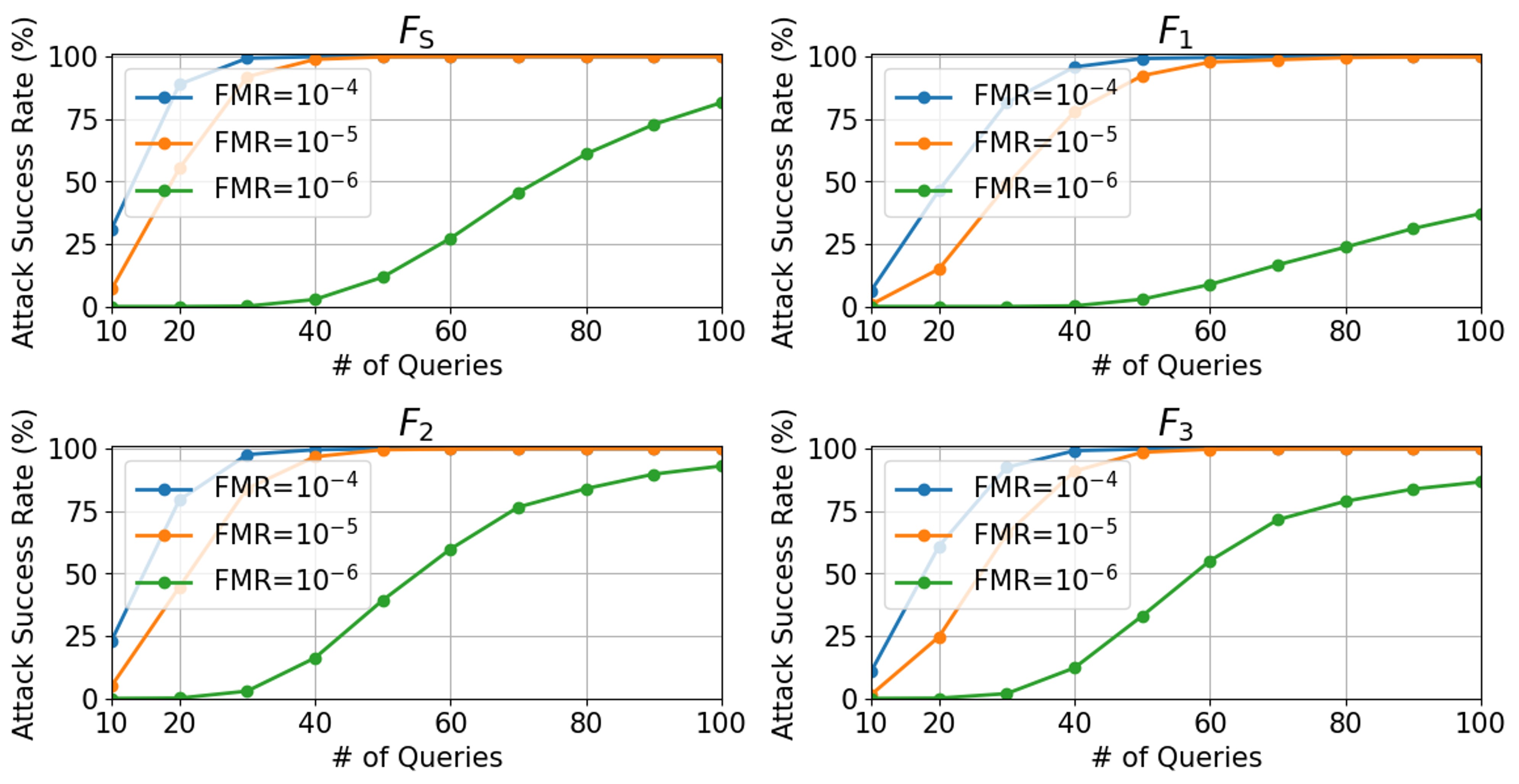}
        \vspace{-7pt}
        \caption{CFP-FP}
    \end{subfigure}
    \vfill
    \begin{subfigure}[t]{\linewidth}
        \centering        
        \includegraphics[width=.95\linewidth]{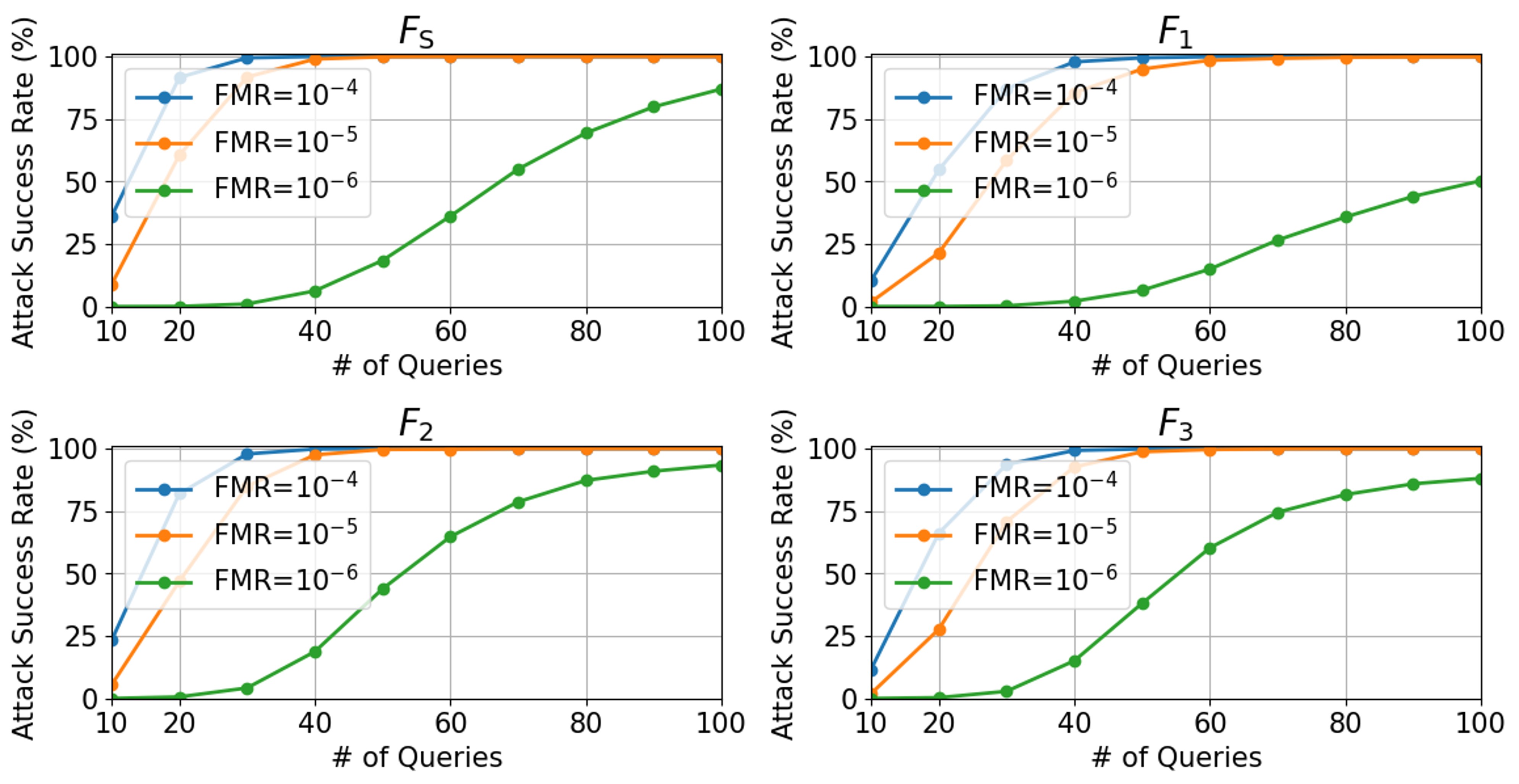}
        \vspace{-7pt}
        \caption{AgeDB}
    \end{subfigure}        
    \vspace{-10pt}
    \caption{Query sweep analysis results on open-source FRSs. Best viewed in color.}
    \label{fig:abstudy_query}
\end{figure}

\begin{figure}
    \centering
    \begin{subfigure}[t]{\linewidth}
        \centering        
        \includegraphics[width=.95\linewidth]{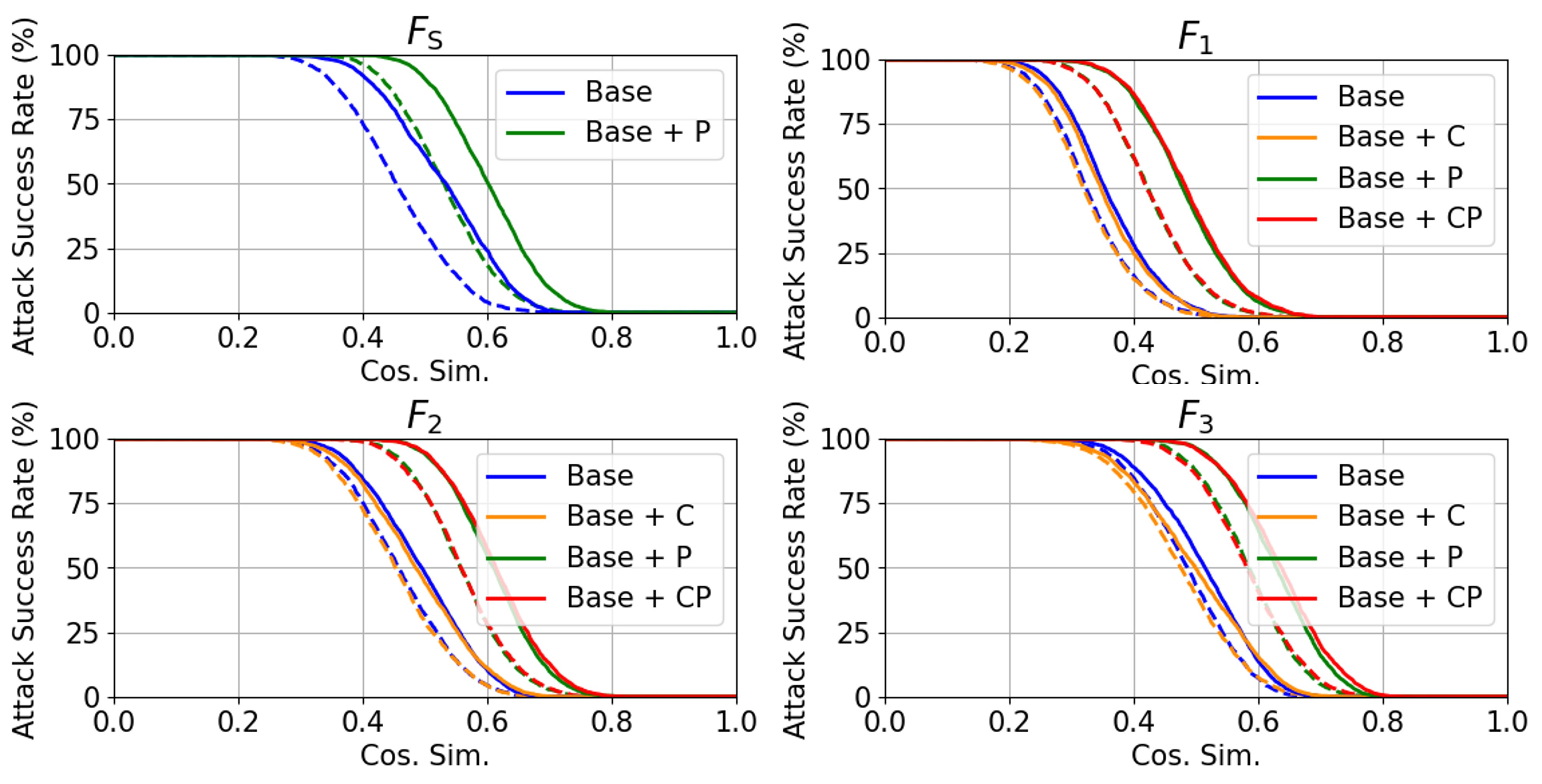}
        \vspace{-7pt}
        \caption{LFW}
    \end{subfigure}
    \vfill
    \begin{subfigure}[t]{\linewidth}
        \centering        
        \includegraphics[width=.95\linewidth]{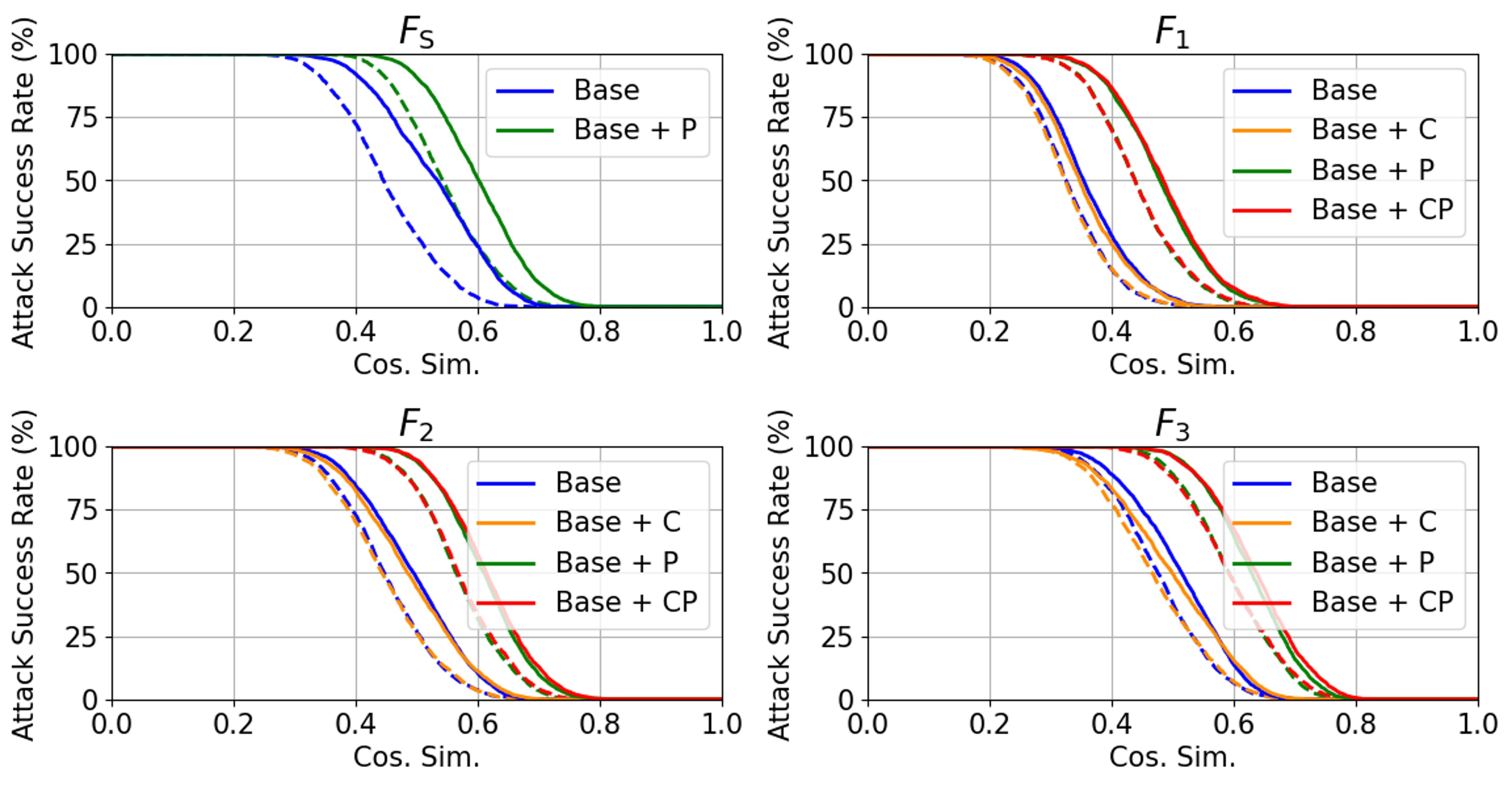}
        \vspace{-7pt}
        \caption{CFP-FP}
    \end{subfigure}
    \vfill
    \begin{subfigure}[t]{\linewidth}
        \centering        
        \includegraphics[width=.95\linewidth]{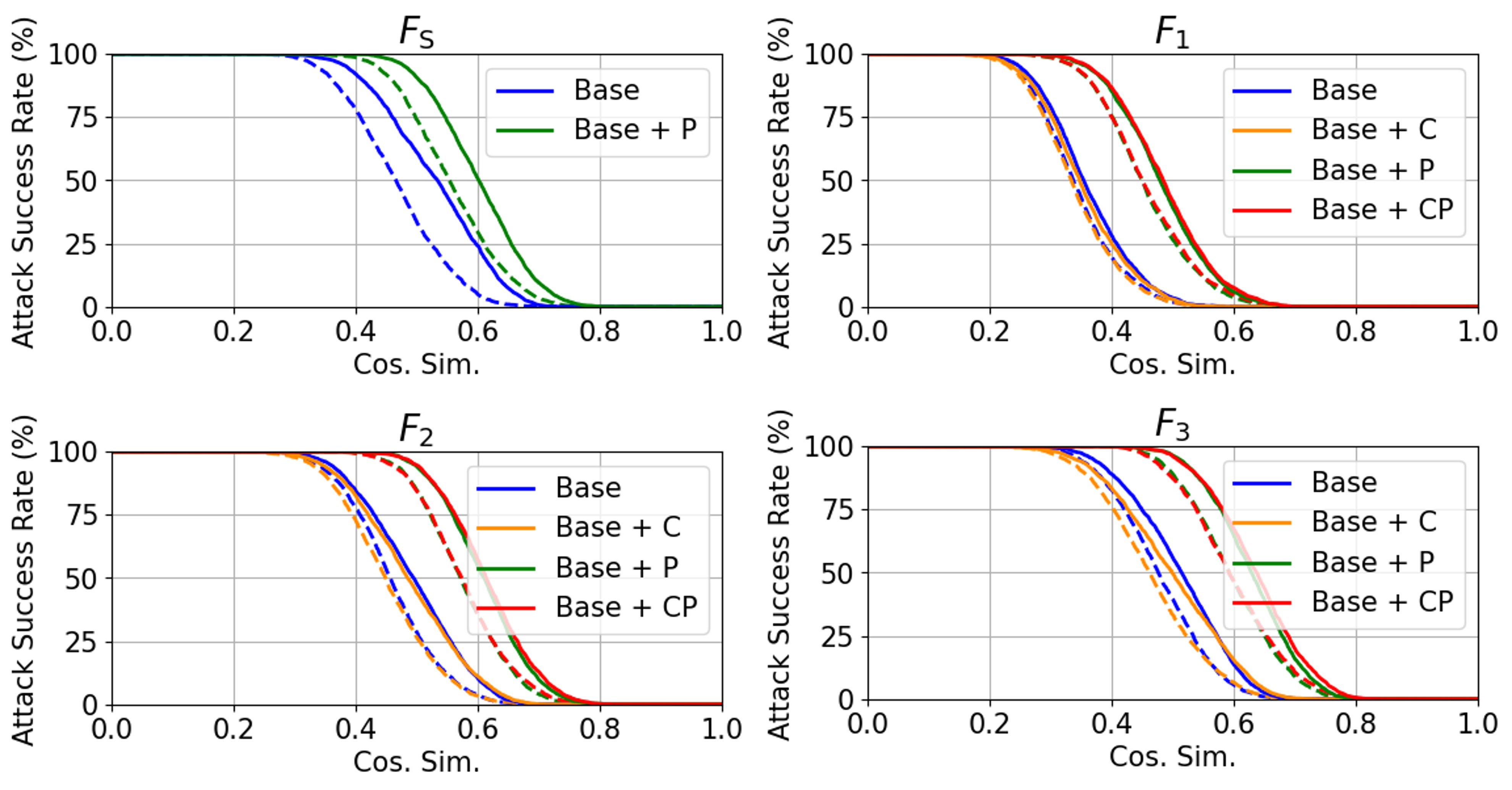}
        \vspace{-7pt}
        \caption{AgeDB}
    \end{subfigure}        
    \vspace{-10pt}
    \caption{Ablation study results on open-source FRSs. Solid/dashed lines denote that SPNet and NbNet are used as inverse models, respectively. Best viewed in color.}
    \label{fig:abstudy_opensource}
\end{figure}

\subsection{Additional Attack Visualization Results}
Fig.~\ref{fig:all_aws_att_img_lfw} visualizes the faces reconstructed from score-based attacks against ${F}_{\mathsf{A}}$. 
Target images are shown in the first row, and the rows below show reconstructed attack images obtained using the baseline and different inverse models and attack configurations. 
All target images are sampled from the LFW dataset.
We can observe that the faces recovered from the baseline attack barely surpass the strict threshold (99), even at the default threshold of 80.
As we progressively apply our technique, the ASR drastically increases, even when our new inverse model is not yet employed.
When combining the correction matrix and PCA-based OFS with the Arc2Face~\cite{papantoniou2024arc2face}, almost half of the samples exceed the high-security threshold.
Finally, with the proposed SPNet, all the recovered faces are authenticated as the targets.
Notably, all the recovered faces barely resemble the target ones, though the confidence scores are substantially different.
This stems from the fact that all the attacks are designed to enforce closeness over the embedding space, rather than the image space from the human's view.

%

\begin{figure}[t]
  \centering
  \includegraphics[width=\linewidth]{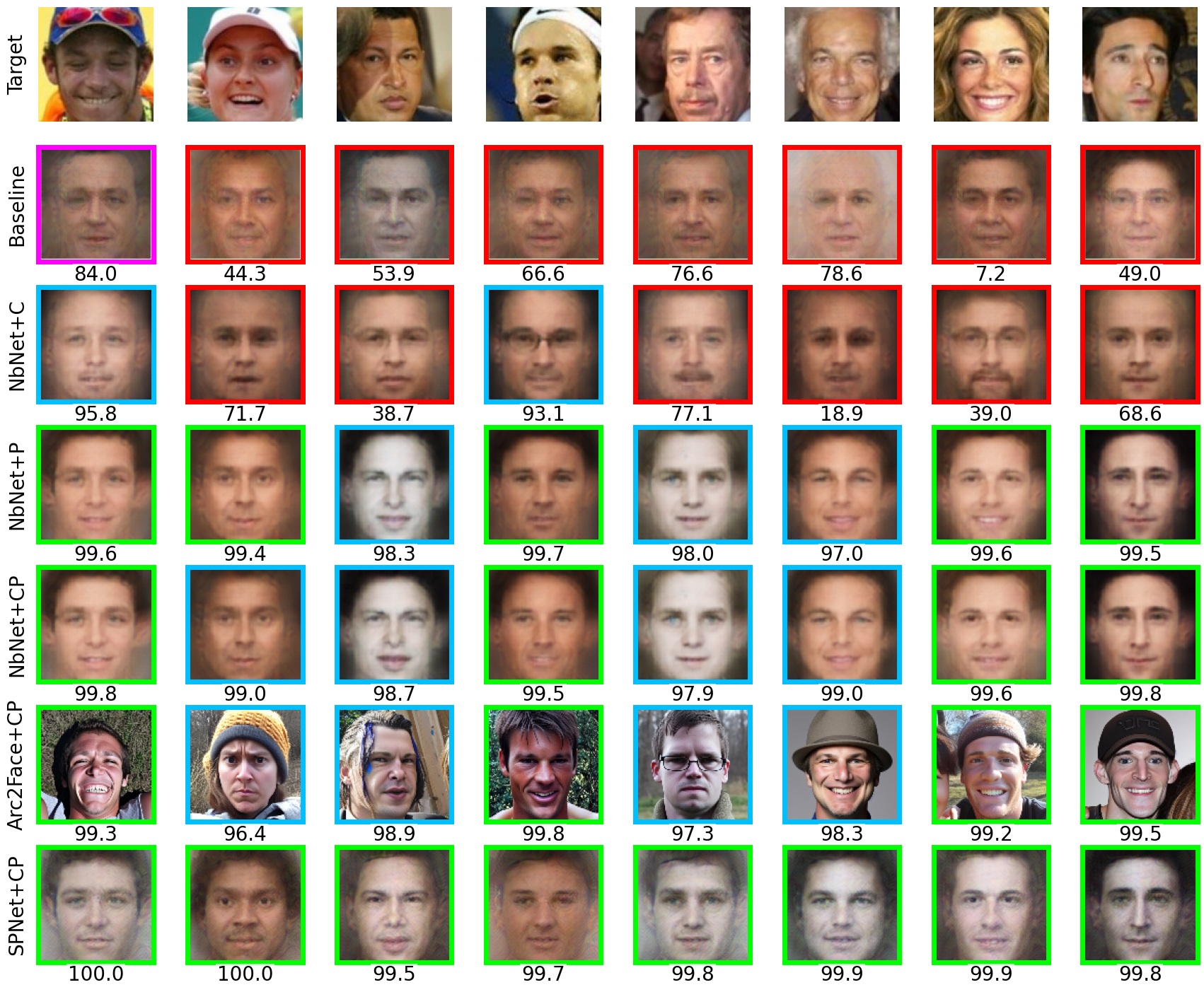}
  \vspace{-15pt}
  \caption{Attack results against $F_{\mathsf{A}}$ on the LFW dataset.
  Box colors indicate the confidence score against the target image in the first row: green for scores $> 99$, skyblue for scores $> 90$, purple for scores $> 80$, and red otherwise. The number below each reconstructed image denotes the returned confidence score.
  }
  \label{fig:all_aws_att_img_lfw}
  \vspace{-10pt}
\end{figure}



\subsection{Additional Baseline Comparison}
We consider three additional score-based attacks~\cite{razzhigaev2021darker,razzhigaev2025inverting,park2023towards} against $F_{\mathsf{S}}$ and assess their ASRs under a 100-query constraint. For~\cite{razzhigaev2021darker}, we refer to the query-wise ASR analysis reported in~\cite{kim2024scores}, where the curve remains clearly below 10\% at the best-accuracy threshold even with around 3,000 queries. For~\cite{razzhigaev2025inverting}, after zero initialization, we perform optimization using 100 score queries and observe an ASR of 0.49 $\pm$ 0.02\% at the best-accuracy threshold over 10 runs. For~\cite{park2023towards}, which allocates 2,000 queries each to initialization and reconstruction, we follow the same ratio under the 100-query budget using 50 initialization and 50 reconstruction queries, yielding an ASR of 1.22 $\pm$ 0.02\% at $\mathrm{FMR}=10^{-4}$ over three runs. 

Based on these observations, these attacks are not competitive with our attack under the 100-query budget at $\mathrm{FMR}=10^{-6}$. Public implementations are available for~\cite{razzhigaev2025inverting,park2023towards}.
\footnote{Public implementations of~\cite{park2023towards} and~\cite{razzhigaev2025inverting} are available at \url{https://github.com/1ho0jin1/Black-box-Face-Reconstruction} and \url{https://github.com/FusionBrainLab/AdversarialFaces}, respectively.}

\end{document}